\documentclass[lettersize,journal]{IEEEtran}

\usepackage[colorlinks=true,
            linkcolor=blue,
            citecolor=green,
            urlcolor=magenta]{hyperref}

\usepackage{makecell}

\ifCLASSINFOpdf
   \usepackage[pdftex]{graphicx}
\else
   \usepackage[dvips]{graphicx}
\fi
\usepackage{amsmath}
\usepackage{amssymb}
\usepackage{amsthm}

\usepackage[]{footmisc}

\usepackage{multirow}

\usepackage{acro}
\DeclareAcronym{ire}{
  short = IRE,
  long  = interactive robotic exploration,
}

\DeclareAcronym{psm}{
  short = PSM,
  long  = patient-side manipulator,
}

\DeclareAcronym{ecm}{
  short = ECM,
  long  = endoscopic camera manipulator,
}
\DeclareAcronym{dof}{
  short = DoF,
  long  = degree of freedom,
  long-plural = degrees of freedom,
  short-plural = DoFs 
}

\DeclareAcronym{ltv}{
    short = LTV,
    long = latent topology variable
}

\DeclareAcronym{ramis}{
  short = RAMIS,
  long  = robot-assisted minimally invasive surgery,
}

\DeclareAcronym{dvrk}{
    short = dVRK,
    long  = \textit{da Vinci} research kit,
    sort  = dVRK
}

\DeclareAcronym{scm}{
    short = SCM,
    long  = structural causal model
}

\DeclareAcronym{dom}{
    short = DOM,
    long  = deformable object manipulation
}

\DeclareAcronym{ipp}{
    short = IPP,
    long  = informative path planning
}

\DeclareAcronym{fem}{
    short = FEM,
    long  = finite element method
}

\DeclareAcronym{sam}{
    short = SAM2,
    long  = Segment Anything Model 2
}

\DeclareAcronym{dtc}{
    short = DTC,
    long  = deformation transmission control
}

\DeclareAcronym{gp}{
    short = GP,
    long  = Gaussian process
}

\DeclareAcronym{mpc}{
    short = MPC,
    long  = model predictive control
}

\DeclareAcronym{ce}{
    short = CE,
    long  = causal effect
}

\DeclareAcronym{eig}{
    short = EIG,
    long  = expected information gain
}

\DeclareAcronym{qp}{
    short = QP,
    long  = quadratic programming
}

\DeclareAcronym{dtr}{
    short = DTR,
    long  = deformation transmission ratio
}

\usepackage{tikz, pgfplots}
\usepackage{pgfplotstable}
\usetikzlibrary{arrows.meta, positioning, shapes.geometric, calc, decorations.pathreplacing}

\definecolor{techblue}{RGB}{0, 113, 188}   % Decoupled Mode
\definecolor{techorange}{RGB}{217, 83, 25} % Blocked Mode
\definecolor{overlap}{RGB}{140, 80, 140}   % 混合区域的紫色示意
\definecolor{axisgrey}{RGB}{80, 80, 80}

\pgfmathdeclarefunction{gauss}{2}{%
  \pgfmathparse{1/(#2*sqrt(2*pi))*exp(-((x-#1)^2)/(2*#2^2))}%
}

\tikzset{
    midarrow/.style={
        decoration={
            markings,
            mark=at position #1 with {\arrow{Stealth[scale=1.2]}}
        },
        postaction={decorate}
    }
}

\usepackage{threeparttable}
\usetikzlibrary{positioning}
\usetikzlibrary{calc}

\newtheorem{thm}{Theorem}
\newtheorem{proposition}[thm]{Proposition}
\usepackage{mathtools}
\usepackage{soul}
\usepackage{siunitx}

\makeatletter
\AtBeginDocument{\let\hl\@firstofone}
\makeatother

\DeclareMathOperator*{\argmin}{arg\,min}
\DeclareMathOperator*{\argmax}{arg\,max}

\usepackage{algpseudocode}
\usepackage{algorithm}

\ifCLASSOPTIONcompsoc \usepackage[caption=false,font=normalsize,labelfont=sf,textfont=sf]{subfig}
\else
 \usepackage[caption=false,font=footnotesize]{subfig}
\fi
\usepackage{siunitx}
\ifdefined\unit\else
  \ifdefined\NewCommandCopy
    \NewCommandCopy\unit\si
  \else
    \NewDocumentCommand\unit{O{}m}{\si[#1]{#2}}
  \fi
\fi

\pgfplotsset{compat=1.16}
\begin{document}
\title{CADeT: Causal-Aware Deformation Transmission for Indirect Robotic Manipulation of Soft Tissue}


\author{Junlei Hu, Dominic Jones~\IEEEmembership{Member,~IEEE}, Pietro Valdastri,~\IEEEmembership{Fellow,~IEEE}
        % <-this % stops a space
\thanks{This work was supported in part by the European Research Council (ERC) through the European Union’s Horizon 2020 Research and Innovation Programme under Grant 818045, in part by the Engineering and Physical Sciences Research Council (EPSRC) under Grant EP/V047914/1, and in part by the National Institute for Health and Care Research (NIHR) Leeds Biomedical Research Centre (BRC) (NIHR203331). }
\thanks{Junlei Hu, Dominic Jones, Pietro Valdastri are with the STORM Lab, Institute of Autonomous Systems and Sensing (IRASS), School of Electronic and Electrical Engineering,
University of Leeds, Leeds, UK. Email:{\tt \{j.hu2, d.p.jones, p.valdastri\}@leeds.ac.uk}}
}

% \thanks{Manuscript received December 1, 2012; revised August 26, 2015. 
% This work is supported by ...Corresponding author: M. Shell (email: http://www.michaelshell.org/contact.html).}}

\markboth{IEEE TRANSACTIONS ON ROBOTICS,~Vol.~~, No.~~, August~2026}%
{Hu \MakeLowercase{\textit{et al.}}: A Sample Article Using IEEEtran.cls for IEEE Journals}



% for Transactions on Magnetics papers, we must declare the abstract and
% index terms PRIOR to the title within the \IEEEtitleabstractindextext
% IEEEtran command as these need to go into the title area created by
\maketitle

\begin{abstract}
Indirect manipulation of deep-seated deformable anatomy \hl{inaccessible to the robot} is challenging in \ac{ramis} \hl{because intervening tissues spatially filter deformation
transmission. 
Passive observations can be ambiguous because the Decoupled and Blocked modes may produce similar motion responses.}
We propose CADeT, a causal-aware deformation transmission framework that integrates \ac{scm} with active sensing
\hl{to infer a latent transmission mode and estimate a state-dependent adhesion Jacobian online. During normal manipulation, control actions update the mode belief; when ambiguity persists, an additional probing action is selected to improve mode distinguishability.}
\hl{The mode belief and learned Jacobian are incorporated into a belief-aware model predictive controller for indirect target-shape control.}
\hl{Validation in simulation and on the} \ac{dvrk}\hl{, using phantom and ex vivo porcine tissues, shows higher mode-identification accuracy and faster shape-error convergence than the evaluated model-free and model-based baselines.}
\hl{These results show that active sensing improves mode identification and indirect deformation control under the evaluated conditions.}\end{abstract}

% Note that keywords are not normally used for peer review papers.


\IEEEdisplaynontitleabstractindextext

% \fi
\IEEEpeerreviewmaketitle

\acresetall

\section{Introduction}
\IEEEPARstart{R}{OBOT-ASSISTED} minimally invasive surgery (RAMIS) is progressively evolving from resection-centric tasks toward complex exploration and reconstruction procedures that require the manipulation of deep-seated anatomical structures \cite{dupont2021decade}. 
Direct manipulation of deep or delicate organs is often infeasible due to anatomical constraints and injury risk.
\hl{In pancreatic exploration or mesenteric handling, the target may lie beneath or behind more accessible tissues while remaining partially visible for deformation tracking.}
To expose or manipulate these targets, surgeons must manipulate accessible proxy organs to transmit forces to the target anatomy indirectly.
\hl{This indirect-manipulation paradigm differs from classical robotic manipulation, where the end-effector directly contacts the object of interest.}
While direct \ac{dom} has been extensively studied, \hl{the intervening medium creates a compound mechanical system that challenges standard modelling approaches} \cite{yin2021modeling}.

The fundamental physical challenge in indirect manipulation arises from deformation transmission through the intermediary proxy tissue.
Unlike rigid mechanisms, soft biological tissue attenuates localized deformation as it propagates through the viscoelastic medium and can therefore be viewed as a spatial low-pass filter, consistent with Saint-Venant's Principle \cite{timoshenko2012theory}.
Consequently, surface actuation produces smoother and weaker responses at the target, reducing the effective control authority and making fine-grained shape servoing more challenging \cite{hu2023occlusion, yang2023model}.

\begin{figure*}[!hbt]
    \centering
    \addtolength{\abovecaptionskip}{-10pt}
    \includegraphics[width=0.9\linewidth]{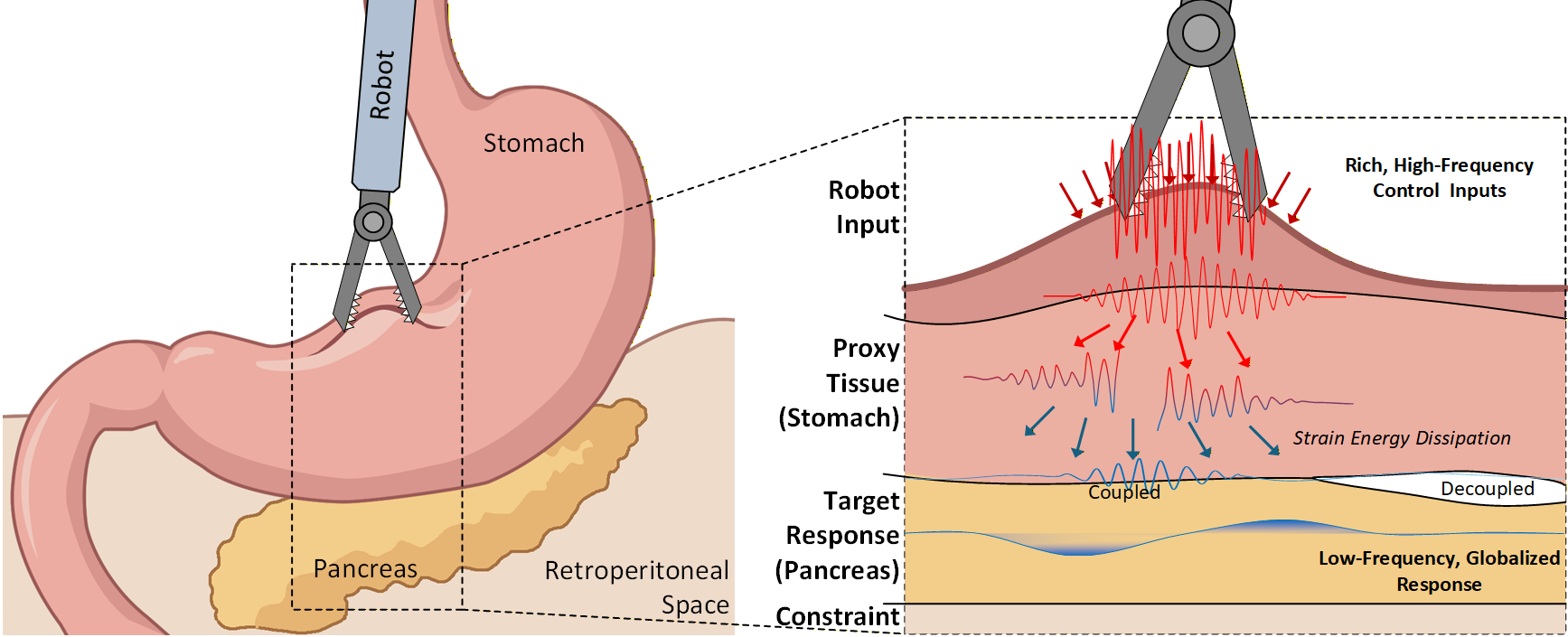}
    \caption{Conceptual illustration of indirect deformation transmission. High-frequency robotic inputs at the proxy surface (stomach) are attenuated by viscoelastic strain-energy dissipation, producing a low-frequency, globalized deformation whose transmission to the target organ (pancreas) depends on the latent transmission mode.}
    \label{fig:figure1}
\end{figure*}

More critically, this spatial filtering introduces an \hl{observational ambiguity} within the perception--action loop.
Due to the attenuation of distinctive high-frequency geometric and haptic features, the observed tissue deformation may no longer uniquely indicate the underlying physical interaction.
For example, \hl{limited or absent motion of the visible target region may indicate either ineffective proxy--target transmission or blockage by surrounding anatomy, such as the retroperitoneal wall.
Although these conditions produce similar visual responses, they require different control actions: continued manipulation may retract a coupled target but increase tissue compression when the target is blocked.}

Existing \ac{dom} methods, including \ac{fem} simulations
\cite{saghour2025dual, han2025quasi} and model-free Jacobian estimation
\cite{navarro2016automatic, zhu2021vision}, typically assume a continuous and globally consistent mapping between actuator motion and object deformation.
\hl{In indirect manipulation, however, this relationship can become one-to-many and non-smooth as contact states switch, such as through adhesion failure or collision.
These methods do not explicitly model the latent transmission mode, so transitions between unconstrained and contact-constrained regimes can result in model mismatch or undesired interactions.}

To resolve this \hl{observational ambiguity}, we propose CADeT, a framework for
\textbf{C}ausal-\textbf{A}ware \textbf{De}formation
\textbf{T}ransmission \hl{that integrates} \ac{scm}
\hl{with active sensing}.
\hl{The SCM represents how the latent transmission mode modulates the proxy-to-target deformation pathway and provides the structural basis for the mode-conditioned observation models.
These models are used in a Bayesian update to maintain the mode belief.}
\hl{CADeT continuously updates the mode belief using tissue responses induced by normal control actions. 
When these observations do not distinguish the Decoupled and Blocked modes, the robot applies a small additional probing action. 
Within the SCM, this probing action is treated as an intervention, allowing its predicted response to be compared across the competing modes.}
Building on this inference, we introduce a \ac{dtc} strategy that uses the mode belief and learned coupling model to regulate target deformation through the proxy organ.
The main contributions of this paper are:
\begin{enumerate} 
    \item \hl{An} \ac{scm}\hl{-based transmission-mode inference framework for indirect soft-tissue manipulation. We show that passive observation may be insufficient to distinguish the Decoupled and Blocked modes, and use active sensing with additional probing actions to infer the latent transmission mode online.} 
    \item A \ac{dtc} scheme for indirect shape control that integrates the mode belief and a \ac{gp} estimate of
the adhesion Jacobian within a model predictive control framework. \hl{The controller weights the predicted deformation transmission by the belief in the Coupled Mode and uses the resulting model to plan the proxy motion.}
    \item Experimental validation on the \ac{dvrk} using both phantom and ex vivo tissues, demonstrating transmission-mode identification and indirect shape control under different interaction conditions.
\end{enumerate}

\section{Related Work}

\subsection{Modelling and Control of Deformable Objects}
Effective manipulation of soft tissues requires a reliable mapping between robot actuation and tissue deformation.
Existing approaches can be broadly categorized into physically based and data-driven methods \cite{yin2021modeling}.

Physically based methods, such as \ac{fem} \cite{saghour2025dual, koessler2021efficient} or mass-spring systems \cite{kita2011clothes, makiyeh2023shape}, model tissue deformation from mechanical principles.
These methods can achieve high accuracy but require tissue properties, such as stiffness or elasticity, that are difficult to measure accurately in vivo.
Although real-time simulation frameworks (e.g., SOFA \cite{faure:hal-00681539}) have improved computational efficiency, complex interactions, particularly tissue collisions, can remain too computationally demanding for fast robotic control loops.

Data-driven \cite{hu20193} and model-free methods \cite{hu2023occlusion, yang2023model, navarro2016automatic, lagneau2020active, zhu2021vision, shetab2023lattice, hu2026multiscale}, on the other hand, avoid complex physical models.
Many model-free approaches estimate a deformation Jacobian online to map robot motion to tissue deformation \cite{navarro2016automatic, lagneau2020active, zhu2021vision, shetab2023lattice}.
Classical numerical approaches, such as the Broyden update \cite{navarro2016automatic}, incrementally update this relationship during manipulation.

More recently, learning-based approaches have been employed to model complex deformations.
Deep Neural Networks \cite{hu20193} have been used to model global tissue dynamics, while Gaussian Process Regression (GPR) \cite{hu2018three} can predict deformation while estimating uncertainty.
Shape representations have also evolved from simple surface points \cite{hu2023occlusion} to more advanced features such as Fourier descriptors \cite{navarro2017fourier} or latent representations \cite{zhou2021lasesom}.

\hl{Learning-based control policies have also been explored for} \ac{dom}.
\hl{Reinforcement learning has been used to learn manipulation policies from sensory interaction} \cite{zhao2025learning}, \hl{while demonstration-enhanced deep reinforcement learning has been applied to deformable-object manipulation tasks} \cite{wang2025robot}. 
\hl{Imitation learning has also been investigated for dual-arm fabric manipulation} \cite{zhu2025dual}, and learning-based predictive control has been developed for constrained deformable linear object manipulation \cite{tang2024learning}.

\subsection{Indirect Manipulation and Transmission Dynamics}
\hl{Indirect manipulation through intermediary bodies has been studied across robotic domains.
In microrobotics, optical tweezers or magnetic fields can transport micro-scale objects through fluid media} \cite{thakur2014indirect}.
\hl{Such media are typically homogeneous and well modelled, yielding relatively stable and predictable transmission.}

Indirect manipulation has also been investigated in macroscopic robotic manipulation, particularly for rigid objects. 
Representative examples include pushing-based manipulation, where a robot moves an object by applying forces at non-grasped contact points \cite{hogan2020reactive, hogan2018reactive, suresh2021tactile, sodhi2021learning}, and tool-mediated manipulation, where an external tool or a support surface serves as a mechanical intermediary between the robot and the target object \cite{kim2022active, shirai2023tactile}.
In these scenarios, motion or force transmission is governed mainly by rigid-body contact mechanics and friction, while contact modes are often discrete and explicitly modelled \cite{hogan2020reactive, doshi2022manipulation}.
As a result, interaction-state changes, such as sticking, sliding, or contact loss, can often be inferred from force or motion measurements.

\hl{In contrast, surgical indirect manipulation transmits motion through soft, heterogeneous tissues acting as spatial low-pass filters} \cite{timoshenko2012theory}.
\hl{Proxy-surface motions and forces are attenuated during propagation, reducing amplitude and high-frequency cues at the target and limiting fine control and haptic feedback.}

Related work on manipulation under constraints has investigated interaction with confined environments or obstacle avoidance \cite{huang2023deformable}. However, constraints are typically treated as static and known, and changes in contact state are not explicitly modelled. 
In indirect manipulation, the key difficulty lies in distinguishing between qualitatively different interaction modes, such as loss of transmission between proxy and target versus a target that is immobilized by an external constraint. 
From a visual standpoint, both scenarios may manifest as an absence of target motion, making them difficult to disambiguate using standard controllers.

\subsection{Active Sensing and Causal Inference}
When passive observation is insufficient to characterize complex and unstructured environments, robots can actively interact with them to acquire informative sensory data.
In surgical robotics, active perception has been explored for mechanical property identification, geometric exploration, and model refinement.

Mechanical property identification uses controlled interactions to estimate unobservable physical parameters.
Robotic palpation has been extensively studied for tissue-stiffness and friction estimation \cite{nichols2015methods, yan2021fast}, while active tactile exploration has enabled probabilistic classification of soft-tissue inclusions \cite{scimeca2022action}.
Active manipulation has also been used to resolve geometric ambiguity.
For instance, Shinde et al. \cite{shinde2024jiggle} used active probing to estimate deformable-tissue boundary parameters, while dynamic deformation tracking has maintained perception during tool--tissue contact \cite{li2020super}.
Related strategies include autonomous tissue retraction to reveal occluded anatomical structures \cite{attanasio2020autonomous} and exploratory motions to distinguish foreground from background tissues.

Active perception has further been used for information-theoretic model refinement in control tasks \cite{hu2026autonomous}.
In visual servoing and deformation control, exploratory actions maximize information gain about uncertain model parameters, such as the deformation Jacobian, thereby improving convergence and robustness when initial models are inaccurate.

While these methods address parametric uncertainty, such as unknown stiffness, and geometric uncertainty, such as unknown boundaries, they generally assume a fixed interaction structure.
Structural uncertainty, such as whether a proxy organ is physically coupled to the target, is rarely considered and requires reasoning beyond parameter estimation.

Causal modelling provides a principled framework for reasoning about such \hl{hidden interaction mechanisms}.
Building on Pearl's theory of causality \cite{pearl2009causality}, recent robotic studies have applied causal inference to fault diagnosis and tool-use reasoning, primarily at the planning or task level on rigid objects \cite{ahmed2020causalworld, xia2025cage, lee2021causal}.
In contrast, our framework represents probing actions as interventions within the \ac{scm} and uses the resulting tissue responses for online mode inference during low-level deformable-object control.
The resulting mode belief distinguishes the Decoupled and Blocked modes and allows the controller to adapt to changing interaction conditions.

\section{Problem Formulation}
We consider a surgical scene in which a vision-guided multi-arm robot manipulates a proxy organ (e.g., stomach) to indirectly deform a deep-seated target organ (e.g., pancreas) in \ac{ramis}. 
\hl{The target is not directly manipulated by the robot but is assumed to remain at least partially visible in the endoscopic view, with sufficient visual features available for tracking.}
The goal is to control the proxy organ to achieve a desired deformation of \hl{the visible target region while accounting for possible environmental constraints}.
\hl{The observable proxy and target configurations at discrete time} $t$ are
denoted by $\mathbf{g}_t \in \mathbb{R}^{n_\mathrm{g}}$ and
$\mathbf{p}_t \in \mathbb{R}^{n_\mathrm{p}}$, \hl{respectively, and are represented by visual features extracted from the visible tissue regions.
Accordingly,} $\mathbf{p}_t$ \hl{represents the observed target region rather than the complete organ geometry. Completely occluded targets are not considered in the current formulation.}
The relationships between the camera and robot, as well as between the robot and the proxy organ, are calibrated before manipulation, as shown in Fig.~\ref{fig:placeholder}.

\begin{figure}[t]
    \centering
    \addtolength{\abovecaptionskip}{-10pt}
    \input{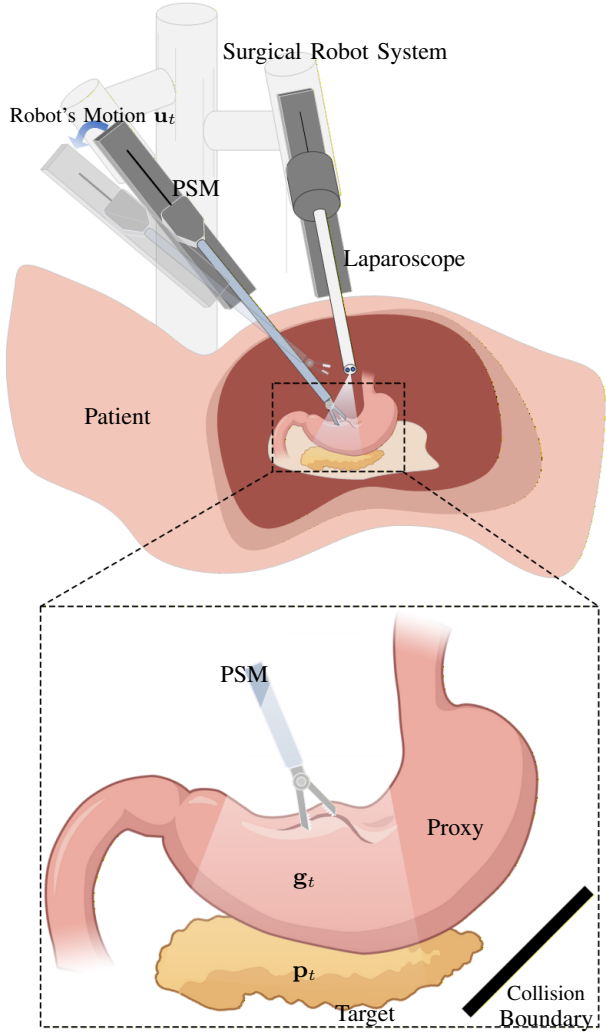}
    \caption{Kinematic configuration for indirect manipulation in \acs{ramis}. 
    The robot actuates the proxy tissue while the target state is observed visually, with an external constraint representing the Blocked Mode.}
    \label{fig:placeholder}
\end{figure}

The robotic system directly manipulates only the proxy organ. 
\hl{We assume that a stable grasp on the proxy tissue has been established before autonomous manipulation. Grasp planning is beyond the scope of this work, and CADeT performs local deformation control under the current grasp.}

Let $\mathbf{u}_t \in \mathbb{R}^{m}$ denote the \hl{incremental Cartesian position command applied to the robot end-effector over one control interval}.
\hl{The evolution of the proxy organ is governed by its viscoelastic dynamics and the applied control action:}
\begin{equation} \label{eq:proxy_dynamics}
    \mathbf{g}_{t+1} = f_\mathrm{g}(\mathbf{g}_t, \mathbf{u}_t) + \mathbf{w}_t^\mathrm{g} 
\end{equation}
where $\mathbf{w}_t^{\mathrm g}\in\mathbb{R}^{n_{\mathrm g}}$ represents bounded stochastic disturbances from respiration and cardiac motion.
Crucially, the target state $\mathbf{p}_t$ is under-actuated; its motion is not directly controlled but induced by deformation transmission from $\mathbf{g}_t$.

\subsection{Hybrid Deformation Transmission Dynamics}
The fundamental challenge in indirect manipulation is modelling the complex mapping from the proxy organ state to the target response.
Unlike direct manipulation, where the end-effector maintains a known mechanical connection with the object and gripper motion therefore induces deformation, indirect manipulation relies on an uncertain and potentially intermittent transmission chain.

We model this process using a hybrid transmission function $\Phi$, which depends on continuous coupling parameters $\theta \in \mathbb{R}^{d}$ (e.g., local stiffness or deformation Jacobian) and a discrete, unobservable latent transmission mode $\tau_t \in \{0,1,2\}$ that governs the transmission behaviour.
The transmission mode is assumed to remain locally persistent over consecutive control cycles, but may change as the local interaction condition changes, such as with the onset or release of an external constraint.
The target organ dynamics are formulated as
\begin{equation} \label{eq:target_dynamics}
\mathbf{p}_{t+1} = \Phi(\mathbf{p}_t, \mathbf{g}_t, \dots)
= \mathbf{p}_t + \mathcal{T}(\mathbf{g}_t, \mathbf{g}_{t+1}, \mathbf{u}_t; \theta, \tau_t) + \mathbf{w}_t^\mathrm{p},
\end{equation}
where $\mathcal{T}(\bullet)$ denotes the deformation transmission term, parametrized by $\theta$ and conditioned on $\tau_t$, while $\mathbf{w}_t^\mathrm{p}$ represents bounded, unmodelled disturbances.

The system evolves according to three modes. 
First, in the \emph{Decoupled Mode} ($\tau_t = 0$), adhesion between the proxy and target is insufficient or broken. 
Consequently, the mechanical transmission chain is effectively severed; deformation of the proxy induces negligible motion in the target, characterized by 
\[\mathcal{T}( \bullet \mid \tau_t = 0) \approx \mathbf{0}.\] 
In this state, the target remains stationary while the proxy experiences minimal reaction forces.

Conversely, under the \emph{Coupled Mode} ($\tau_t = 1$), the organs are mechanically linked via adhesion, and deformation is transmitted through a viscoelastic coupling. 
This transmission is modelled as
\begin{equation} \label{eq:tau=1}
    \mathcal{T}( \bullet \mid \tau_t = 1) \approx \mathbf{J}_{\mathrm{adh}}(\mathbf{g}_t; \theta) \cdot (\mathbf{g}_{t+1} - \mathbf{g}_t),
\end{equation}
where $\mathbf{J}_{\mathrm{adh}} \in \mathbb{R}^{n_\mathrm{p} \times n_\mathrm{g}}$ is a state-dependent adhesion Jacobian characterizing the instantaneous deformation transfer.

\hl{Finally, the system may enter a Blocked Mode} $(\tau_t=2)$,
\hl{where mechanical coupling remains present but an external constraint
locally restricts deformation transmission along the current manipulation
direction. The resulting target response can therefore become negligible
despite continued proxy deformation, such that}
\[
\mathcal{T}(\bullet \mid \tau_t=2) \approx \mathbf{0}.
\]
\hl{The transmission mode describes the local response under the current interaction direction; therefore, the system may transition between the Coupled and Blocked Modes as the manipulation direction or constraint condition changes.}
This formulation explicitly captures the response bifurcation inherent in indirect manipulation: the same proxy deformation
$\Delta\mathbf{g}_t=\mathbf{g}_{t+1}-\mathbf{g}_t$ may induce non-negligible target motion, or none at all, depending on the latent transmission mode.

\subsection{Observational Ambiguity under Spatial Filtering} \label{sec:casual_ambiguity}
The modelling challenge is further compounded by the spatial low-pass filtering effect of soft proxy tissue. 
High-frequency interaction cues, such as localized stress concentrations caused by collisions, are rapidly attenuated as deformation propagates through the viscoelastic medium of the proxy organ. Consequently, different transmission modes can produce nearly identical kinematic observations at the target.

Consider the degenerate case of a stationary target, i.e., $\Delta \mathbf{p}_t \approx \mathbf{0}$. This observation is ambiguous; it may stem from a Decoupled Mode ($\tau_t = 0$), where the proxy is detached and transmits minimal force, or a Blocked Mode ($\tau_t = 2$), where the target is rigidly constrained by surrounding anatomy. 
Although the underlying mechanics differ drastically, passive observation alone is insufficient to distinguish between these causes. 
A naive regression model learning a direct mapping $\mathbf{p}_{t+1} = \hat{\Phi}(\mathbf{g}_t)$ would conflate these regimes, leading to misleading predictions and potentially unsafe control decisions.

We refer to this phenomenon as observational ambiguity, formally described as the non-injectivity of the mapping from the interaction history $\mathcal{H}_t = \{ \mathbf{g}_{0:t}, \mathbf{p}_{0:t}, \mathbf{u}_{0:t-1} \}$ to the latent transmission mode $\tau_t$.
Due to this ambiguity, the robot's estimate of the transmission mode remains highly uncertain.
We quantify this uncertainty using the Shannon entropy $H(\tau_t \mid \mathcal{H}_t)$, with ambiguity indicated by $H(\tau_t \mid \mathcal{H}_t) > H_{\mathrm{th}}$, where $H_{\mathrm{th}}=$ \SI{0.25}{nat} in all experiments.

\hl{To explicitly represent this uncertainty, we maintain the mode belief}
\begin{equation} \label{eq:posterior_belief}
b_t(\tau)
=
P(\tau_t \mid \mathcal{H}_t),
\end{equation}
\hl{while the continuous coupling uncertainty is represented separately by the} \ac{gp} \hl{posterior over the adhesion Jacobian.
The entropy of} $b_t(\tau)$ \hl{quantifies uncertainty in the transmission mode and is used in Section}~\ref{sec:active_intervent} \hl{to trigger and design probing actions.}

\subsection{Control Objective} \label{sec:control_objective}
The manipulation objective is to compute a finite-horizon control sequence $\mathbf{u}_{t:t+T-1}$ that drives the target organ toward a desired configuration $\mathbf{p}_{\rm d}\in\mathbb{R}^{n_\mathrm{p}}$ \hl{under the current grasp and transmission conditions.}
We formulate this problem within a \ac{mpc} framework.
Over a prediction horizon of length $T$, the cumulative cost is defined as
\[
J(\mathbf{u}_{t:t+T-1}) =
\sum_{k=0}^{T-1}
\left(
\bigl\| \hat{\mathbf{p}}_{t+k+1} - \mathbf{p}_{\rm d} \bigr\|^2_{\mathbf{Q}}
+
\bigl\| \mathbf{u}_{t+k} \bigr\|^2_{\mathbf{R}}
\right),
\]
where $\mathbf{Q} \in \mathbb{R}^{n_\mathrm{p} \times n_\mathrm{p}}$
and $\mathbf{R} \in \mathbb{R}^{m \times m}$ are positive-definite
weighting matrices penalizing \hl{the target tracking error and control
effort}, respectively. The weighted norm
$\| \mathbf{x} \|^2_{\mathbf{W}}$ is defined as
$\mathbf{x}^\top \mathbf{W} \mathbf{x}$.

\hl{The predicted target state} $\hat{\mathbf{p}}_{t+k+1}$ \hl{depends on the uncertain transmission mode and learned coupling dynamics.
The optimal control sequence is therefore obtained by minimizing the expected cost under the current mode belief:}
\begin{equation}
\label{eq:objective}
\mathbf{u}_{t:t+T-1}^{*}
=
\operatorname*{arg\,min}_{\mathbf{u}_{t:t+T-1}}
\;
\mathbb{E}_{\tau\sim b_t}
\left[
J(\mathbf{u}_{t:t+T-1})
\right].
\end{equation}
\hl{When the belief is concentrated on the Coupled Mode, the controller primarily uses the learned coupling model to reduce the target tracking error.
When the transmission mode remains uncertain, its belief attenuates the predicted deformation transmission.
Continuous coupling uncertainty is represented separately by the} \ac{gp} \hl{posterior, whose mean is used for real-time MPC prediction, as detailed in Section}~\ref{sec:dtc}.

\section{Structural Causal Modelling} \label{sec:scm}
Having formalized the state space and hybrid deformation transmission, we now construct an \ac{scm} governing the system dynamics.
Conventional kinematic models, such as a deformation Jacobian, primarily describe correlations between robot motion and tissue deformation.
In contrast, an \ac{scm} explicitly represents directed causal mechanisms and the role of latent variables in modulating them.
This structural perspective is critical for distinguishing different physical processes, namely adhesive coupling between organs and motion blockage caused by environmental constraints.

\begin{figure*}[!hbt]
    \centering
    \addtolength{\abovecaptionskip}{-10pt}
    \input{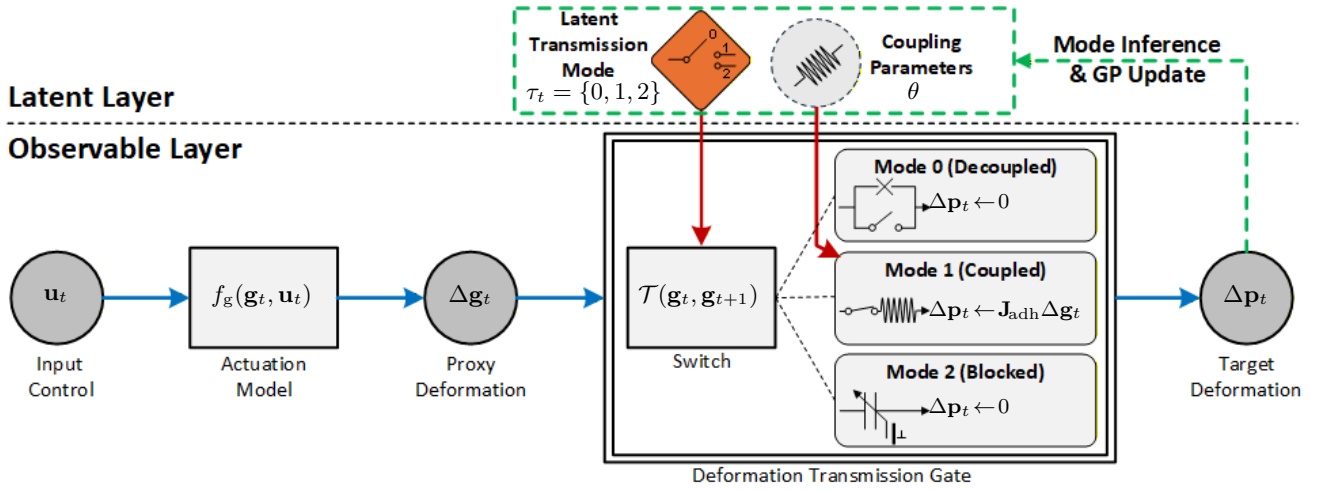}
    \caption{\ac{scm} of indirect deformation transmission.
    The control action $\mathbf{u}_t$ drives proxy deformation $\Delta \mathbf{g}_t$, while the latent transmission mode $\tau_t$ gates the proxy-to-target transmission pathway.}
    \label{fig:diagram}
\end{figure*}

\subsection{The Causal Graph}

As illustrated in the causal diagram, Fig. \ref{fig:diagram}, the evolution of the system is driven by three distinct structural dependencies:
\subsubsection{Actuation Mechanism}
The proxy organ state is directly controlled by the robot, as in \eqref{eq:proxy_dynamics}.

\subsubsection{Continuous Coupling Potential}
A continuous function $\Psi(\mathbf{g}_t, \mathbf{g}_{t+1}; \theta)$ represents the potential deformation transmission between the proxy and target organs.
Parametrised by $\theta$, this function describes the target displacement that would occur under ideal mechanical coupling.
It captures the underlying physical interaction and forms the core continuous component of the transmission term $\mathcal{T}$ defined in \eqref{eq:target_dynamics}.

\subsubsection{The Latent Transmission Mode}
\hl{The effect of proxy motion on the target depends on the latent transmission mode} $\tau_t$.
\hl{We represent this mode-dependent transmission using an indicator function in the structural equation}
\begin{equation}\label{eq:transmission}
    \mathbf{p}_{t+1} := \mathbf{p}_t + \Psi(\mathbf{g}_t, \mathbf{g}_{t+1}; \theta) \cdot \mathbb{I}(\tau_t=1) + \mathbf{w}_t^\mathrm{p}.
\end{equation}
Here, $\mathbb{I}(\cdot)$ \hl{equals 1 in the Coupled Mode and 0 otherwise, activating or suppressing the transmission term.}

This formulation shows how the latent transmission mode determines whether proxy motion affects the target.
When $\tau_t = 1$, corresponding to the Coupled Mode,
\hl{the indicator activates the transmission term} $(\mathbb{I}=1)$, \hl{enabling the causal pathway from proxy motion to target deformation through the coupling function $\Psi$.}
\hl{When $\tau_t=0$, the transmission term is suppressed because the proxy and target are mechanically decoupled. When $\tau_t=2$, the same local model represents suppression of target motion along the currently constrained manipulation direction. 
Importantly, the Blocked Mode does not imply a global loss of proxy--target coupling; the system may return to the Coupled Mode when manipulation proceeds along an unconstrained direction.}
\eqref{eq:transmission} \hl{therefore characterizes the target-side deformation-transmission mechanism. 
The Decoupled and Blocked Modes are further distinguished by their different proxy responses to the applied action, as specified by the mode-conditioned observation model in Section}~\ref{sec:causal_inference}.

\subsection{Non-Parametric Modelling of Coupling Parameters}

The adhesion Jacobian $\mathbf{J}_{\mathrm{adh}}(\mathbf{g})$ in \eqref{eq:tau=1} captures complex, nonlinear, and patient-specific tissue interactions that are not available in closed form.
To model this coupling without imposing restrictive parametric assumptions, we adopt a data-driven formulation based on a \ac{gp} model.

Conditioned on the Coupled Mode, the local transmission dynamics are approximated by first-order linearization of \eqref{eq:tau=1},
\[
\Delta\mathbf{p}_t
\approx
\mathbf{J}_{\mathrm{adh}}(\mathbf{g}_t)\Delta\mathbf{g}_t
+
\boldsymbol{\epsilon}_t,
\qquad
\boldsymbol{\epsilon}_t
\sim
\mathcal{N}(\mathbf{0},\mathbf{R}_{\mathrm{noise}}),
\]
where
$\Delta \mathbf{p}_t=\mathbf{p}_{t+1}-\mathbf{p}_t$
and
$\Delta \mathbf{g}_t=\mathbf{g}_{t+1}-\mathbf{g}_t$.
A fixed sliding window of the most recent $N_\mathrm{J}=10$ Coupled-mode
observations is used to estimate the local adhesion Jacobian.
Specifically,
$\Delta\mathbf{P}_t=[\Delta\mathbf{p}_{t-N_\mathrm{J}+1},\ldots,\Delta\mathbf{p}_t]$
and
$\Delta\mathbf{G}_t=[\Delta\mathbf{g}_{t-N_\mathrm{J}+1},\ldots,\Delta\mathbf{g}_t]$,
from which
$
\hat{\mathbf{J}}_{\mathrm{adh},t}
=
\Delta\mathbf{P}_t\Delta\mathbf{G}_t^{\dagger}.
$
The sliding window maintains a local estimate as the tissue
configuration evolves.
For notational simplicity, the explicit parametrization $\theta$ of
the Jacobian is omitted, and all coupling parameters are represented by the state-dependent function
$\mathbf{J}_{\mathrm{adh}}(\mathbf{g})$.

\hl{Each element of the state-dependent adhesion Jacobian is modelled by an independent} \ac{gp},
\begin{equation}
\label{eq:jacobian_gp}
\begin{aligned}
J_{\mathrm{adh},ij}(\mathbf{g})
&\sim
\mathcal{GP}\!\left(
m_{ij}(\mathbf{g}),
k(\mathbf{g},\mathbf{g}')
\right),\\[-2pt]
&\qquad
i=1,\ldots,n_\mathrm{p},\quad
j=1,\ldots,n_\mathrm{g} .
\end{aligned}
\end{equation}
\hl{A zero prior mean is used before valid local Jacobian estimates are available.}
The kernel function $k(\mathbf{g},\mathbf{g}')$ is chosen as a squared exponential kernel to capture smooth local variations of the coupling with tissue configuration.
Each local estimate $\hat{\mathbf{J}}_{\mathrm{adh},k}$ is paired with the corresponding tissue configuration $\mathbf{g}_k$.
Let $\mathcal{D}_t$ denote the resulting \ac{gp} training pairs.
Only observations for which $b_k(\tau=1)>\beta_{\mathrm{GP}}$ holds for $N_{\mathrm{GP}}$ consecutive updates are used to construct these estimates, with $\beta_{\mathrm{GP}}=0.85$ and $N_{\mathrm{GP}}=5$ in all experiments.

Conditioned on $\mathcal{D}_t$, each \ac{gp} provides the posterior
\begin{equation}
\label{eq:gp_posterior}
J_{\mathrm{adh},ij}(\mathbf{g}_t)
\mid
\mathcal{D}_t
\sim
\mathcal{N}\!\left(
\mu_{ij,t}(\mathbf{g}_t),
\sigma_{ij,t}^{2}(\mathbf{g}_t)
\right).
\end{equation}
The posterior means form the current Jacobian estimate $\boldsymbol{\mu}_t(\mathbf{g}_t)$, while the posterior variances are propagated to the predictive covariance in \eqref{eq:L=1}.

\subsection{Mode-Conditioned Likelihoods} \label{sec:causal_inference}
While the \ac{gp} models the transmission dynamics under the Coupled Mode ($\tau_t = 1$), we also define the physical response patterns distinguishing the three modes. 
The \ac{scm} implies that each mode $\tau_t$ induces a different distribution over the differential observations
$\mathcal{Y}_t = [\Delta \mathbf{g}_t^\top,\Delta \mathbf{p}_t^\top]^\top$.
Accordingly, the mode-conditioned likelihood can be decomposed as 
\[
\mathcal{L}(\mathcal{Y}_t \mid \tau_t,\mathbf{u}_t) = p(\Delta\mathbf{p}_t \mid \Delta\mathbf{g}_t,\tau_t) \,p(\Delta\mathbf{g}_t \mid \mathbf{u}_t,\tau_t). 
\]
\hl{The first factor characterizes deformation transmission from proxy to target, whereas the second characterizes the action-induced proxy response. 
This decomposition is particularly important for distinguishing the Decoupled and Blocked Modes: both can produce negligible target displacement, while an external blocking constraint additionally alters the proxy response.} 
A mode-conditioned likelihood $\mathcal{L}(\mathcal{Y}_t \mid \tau_t,\mathbf{u}_t)$ is therefore defined for each candidate mode and used for Bayesian mode inference.

\hl{The proxy-response distributions are initialized from the first} $N_0$
\hl{probing responses generated by the active sensing procedure.
These observations are temporarily buffered and provide action--response pairs}
$(\mathbf{u}_k,\Delta\mathbf{g}_k)$
\hl{for the current tissue sample and grasping region.
In the considered manipulation setting, these initial responses are collected before an external blocking constraint is activated.
The nominal proxy response is represented by the Gaussian distribution}
\begin{equation*}
P_{\mathrm{n}}(\Delta\mathbf{g}_t \mid \mathbf{u}_t)
=
\mathcal{N}\!\left(
\Delta\mathbf{g}_t;
\boldsymbol{\mu}_{\mathrm{n}}(\mathbf{u}_t),
\boldsymbol{\Sigma}_{\mathrm{n}}
\right),
\end{equation*}
where
$\boldsymbol{\mu}_{\mathrm{n}}(\mathbf{u}_t)
=
\mathbf{A}_{\mathrm{n}}\mathbf{u}_t$,
with $\mathbf{A}_{\mathrm{n}}$ \hl{estimated by least squares from the buffered action--response pairs, and}
$\boldsymbol{\Sigma}_{\mathrm{n}}$
is the empirical covariance of the residuals
$\Delta\mathbf{g}_k-\mathbf{A}_{\mathrm{n}}\mathbf{u}_k$.
\hl{This local initialization accounts for variations between tissue samples and grasping regions without requiring patient-independent material parameters.
The resulting distribution is kept fixed during the subsequent manipulation trial and is reinitialized when the tissue sample or grasping region changes.}

In the Decoupled Mode ($\tau_t = 0$), the target organ is mechanically detached from the proxy.
Consequently, the observed target displacement $\Delta \mathbf{p}_t$ is assumed to be independent of the proxy motion and dominated by measurement noise.
The corresponding likelihood is given by
\begin{equation} \label{eq:L=0}
    \mathcal{L}(\mathcal{Y}_t \mid \tau_t = 0, \mathbf{u}_t) = \mathcal{N}(\Delta \mathbf{p}_t; \mathbf{0}, \mathbf{R}_{\mathrm{noise}}) P_{\mathrm{n}}(\Delta \mathbf{g}_t \mid \mathbf{u}_t),
\end{equation}
where $\mathbf{R}_{\mathrm{noise}}$ denotes the covariance of the effective target-response noise.

When the Coupled Mode is active ($\tau_t = 1$), the target displacement is induced by the proxy deformation through the learned coupling dynamics.
The observation likelihood is derived from the posterior predictive distribution of the \ac{gp} introduced in \eqref{eq:gp_posterior}.
\hl{By marginalizing over the posterior uncertainty of the adhesion Jacobian, the likelihood becomes}
\begin{equation} \label{eq:L=1}
\mathcal{L}(\mathcal{Y}_t \!\mid \!\tau_t = 1, \mathbf{u}_t)
=
\mathcal{N}\!\left(
\Delta \mathbf{p}_t;
\boldsymbol{\mu}_t(\mathbf{g}_t)\Delta \mathbf{g}_t,
\boldsymbol{\Sigma}_{\mathrm{p},t}
\right) \!
P_{\mathrm{n}}(\Delta \mathbf{g}_t \!\mid\! \mathbf{u}_t),
\end{equation}
where $\boldsymbol{\mu}_t(\mathbf{g}_t)$ is assembled from the element-wise posterior means in \eqref{eq:gp_posterior}.
The predictive covariance induced by the \ac{gp} uncertainty is defined as
$
\boldsymbol{\Sigma}^{\mathrm{GP}}_{\mathrm{p},t}
=
\operatorname{Cov}\!\left[
\mathbf{J}_{\mathrm{adh}}(\mathbf{g}_t)
\Delta \mathbf{g}_t
\mid
\mathcal{D}_t
\right] $,
\hl{and the total predictive covariance is}
$
\boldsymbol{\Sigma}_{\mathrm{p},t}
=
\boldsymbol{\Sigma}^{\mathrm{GP}}_{\mathrm{p},t}
+
\mathbf{R}_{\mathrm{noise}}$.
Under the independent element-wise \ac{gp} formulation in \eqref{eq:jacobian_gp},
$\boldsymbol{\Sigma}^{\mathrm{GP}}_{\mathrm{p},t}$ \hl{is obtained from the corresponding posterior variances and represents the uncertainty propagated to the predicted target displacement.}
\hl{The same nominal proxy-response distribution is used in}
\eqref{eq:L=0} and \eqref{eq:L=1},
\hl{since neither the Decoupled nor the Coupled Mode activates the external blocking constraint; these modes are primarily distinguished by the target response.}

The Blocked Mode ($\tau_t = 2$) \hl{corresponds to the case in which the target response is restricted by an external constraint.
Based on target motion alone, this mode may resemble the Decoupled Mode because} $\Delta\mathbf{p}_t$ \hl{can remain small.
However, the external constraint also reduces the proxy response relative to its nominal value.
Rather than learning a separate blocked-response model,} $P_{\mathrm{stiff}}$ \hl{is constructed directly from the locally initialized nominal distribution: }
\begin{equation*}
P_{\mathrm{stiff}}(\Delta\mathbf{g}_t \mid \mathbf{u}_t)
=
\mathcal{N}\!\left(
\Delta\mathbf{g}_t;
\lambda_{\mathrm{b}}
\boldsymbol{\mu}_{\mathrm{n}}(\mathbf{u}_t),
\boldsymbol{\Sigma}_{\mathrm{n}}
\right),
\qquad
0 < \lambda_{\mathrm{b}} < 1,
\end{equation*}
where $\lambda_{\mathrm{b}}$ \hl{is a fixed response-attenuation factor.
A smaller value of $\lambda_{\mathrm{b}}$ represents a stronger reduction in proxy response under the external constraint.
Using the same covariance as the nominal distribution avoids introducing an additional learned model and isolates the effect of the attenuated mean response.
The Blocked-mode likelihood is therefore}
\begin{equation}
\label{eq:L=2}
\mathcal{L}(\mathcal{Y}_t \! \mid \! \tau_t = 2,\mathbf{u}_t)\!
=
\!\mathcal{N}\!\left(
\Delta\mathbf{p}_t;
\mathbf{0},
\mathbf{R}_{\mathrm{noise}}
\right)\!
P_{\mathrm{stiff}}(\Delta\mathbf{g}_t \! \mid \!\mathbf{u}_t).
\end{equation}

\hl{After the proxy-response distributions have been initialized, the buffered observations are evaluated using} \eqref{eq:L=0}--\eqref{eq:L=2} \hl{to initialize the mode belief.
Subsequent responses generated by normal control actions and, when needed, additional probing actions are then processed recursively by the Bayesian mode update.
The parameters of} $P_{\mathrm{n}}$ and $P_{\mathrm{stiff}}$ \hl{remain fixed within each trial, preventing responses generated after blockage from being incorporated into the nominal proxy model.}

\subsection{Passive Observational Indistinguishability}
\hl{To formalize why normal task execution may be insufficient for mode
inference, we analyse the distinguishability of the latent transmission
mode under a passive task-control policy.}

\begin{proposition}[Observational Indistinguishability]
\label{prop:indistinguishability}
Consider the stationary-target regime
$\Delta\mathbf{p}_t\approx\mathbf{0}$.
Under a passive control policy
$\mathbf{u}_t=\pi_{\mathrm{pass}}(\mathcal{H}_t)$
with $\mathbf{u}_t\rightarrow\mathbf{0}$ as the task settles,
the Decoupled and Blocked observation likelihoods become
asymptotically indistinguishable in terms of the
Kullback--Leibler (KL) divergence:
\[
\lim_{t\rightarrow\infty}
D_{\mathrm{KL}}
\!\left(
P(\mathcal{Y}_t\mid\tau=0,\mathbf{u}_t)
\,\|\, 
P(\mathcal{Y}_t\mid\tau=2,\mathbf{u}_t)
\right)
=0.
\]
\end{proposition}
The proof is provided in Appendix~\ref{app:proof_prop1}.

\textit{Remark}: 
Proposition~\ref{prop:indistinguishability} \hl{shows that passive observations may become insufficient to distinguish the Decoupled and Blocked modes as task actions vanish, motivating additional probing to elicit distinguishable responses.}

\section{Transmission-Mode Inference with Active Sensing} \label{sec:active_intervent}
\hl{The mode-conditioned observation models derived from the structural formulation update the transmission-mode belief using tissue responses generated by normal control actions} $\mathbf{u}_t$ \hl{and incorporated into the interaction history} $\mathcal{H}_t$.
\hl{Because these task-directed actions are selected primarily to reduce task error rather than distinguish the transmission modes, their responses may remain ambiguous, particularly when} $\Delta\mathbf{p}_t \approx \mathbf{0}$, \hl{for which the Decoupled} $(\tau_t=0)$ \hl{and Blocked} $(\tau_t=2)$ \hl{modes both predict negligible target motion.}

\hl{When this ambiguity persists, the robot applies an additional probing action} $\delta\mathbf{u}$ \hl{selected to improve mode distinguishability.
Following standard} \ac{scm} intervention notation \cite{pearl2009causality}, \hl{this action is represented as} $\operatorname{do}(\delta\mathbf{u})$, \hl{indicating that it is imposed independently of the nominal task-control policy; the resulting response is then used to update the mode belief.}

\begin{figure}
    \centering
    \addtolength{\abovecaptionskip}{-20pt}
    \input{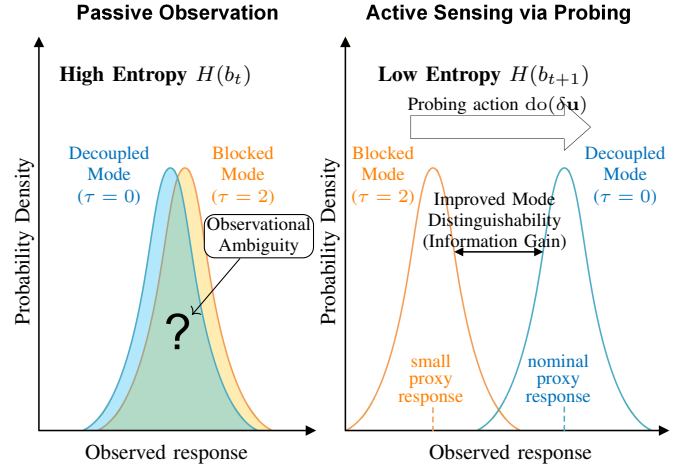}
    \caption{Conceptual illustration of Passive Observation and Active Sensing.}
    \label{fig:ambiguity}
\end{figure}

\subsection{Bayesian Belief Update}
The inference estimates the mode belief
\( b_t(\tau) = P(\tau_t \mid \mathcal{H}_t) \)
over the three transmission modes.
\hl{The transmission mode is assumed locally persistent over consecutive control cycles.}

At each time step $t$, after applying a control input $\mathbf{u}_t$ and observing the outcome $\mathcal{Y}_t$, the robot updates the mode belief recursively using Bayes' rule,
\begin{equation} \label{eq:bayes}
b_{t+1}(k)
=
\eta\,
\mathcal{L}(\mathcal{Y}_t \mid \tau_t = k, \mathbf{u}_t)\,
b_t(k),
\end{equation}
where
\(
\eta =
\left(
\sum_{j=0}^{2}
\mathcal{L}(\mathcal{Y}_t \mid \tau_t = j, \mathbf{u}_t)
b_t(j)
\right)^{-1}
\)
is a normalization constant.
This recursion adopts a local-persistence approximation without an explicit mode-transition model. When the interaction condition changes, subsequent mode-conditioned observations progressively shift the belief toward the newly supported mode.
The likelihood terms
$\mathcal{L}(\mathcal{Y}_t \mid \tau_t, \mathbf{u}_t)$
are computed using the structural indicators defined in
\eqref{eq:L=0}--\eqref{eq:L=2}.

\subsection{Causal Interpretation of Probing Actions}
During task execution, the control input $\mathbf{u}_t$ is selected primarily to minimize task error.
The resulting predictive distribution marginalizes over the unknown transmission mode:
\[
P(\mathcal{Y}_t \mid \mathcal{H}_t,\mathbf{u}_t)
=
\sum_{k=0}^{2}
\mathcal{L}(\mathcal{Y}_t \mid \tau_t = k, \mathbf{u}_t)\,
b_t(k).
\]
\hl{Observational ambiguity occurs} when the component likelihoods overlap, particularly when $\mathcal{L}(\mathcal{Y}_t \mid \tau_t = 0,\mathbf{u}_t)\approx \mathcal{L}(\mathcal{Y}_t \mid \tau_t = 2,\mathbf{u}_t)$, which \hl{limits the mode-belief update}.

\hl{To distinguish these modes, the robot applies an additional probing
action} $\delta\mathbf{u}$, represented in the \ac{scm} as the intervention $\operatorname{do}(\delta\mathbf{u})$. 
\hl{The corresponding mode-conditioned likelihood is}
\[
\mathcal{L}\!\left(
\mathcal{Y}_t
\mid
\tau_t=k,
\operatorname{do}(\delta\mathbf{u})
\right).
\]
\hl{Marginalizing over the current mode belief gives the predictive distribution under the probing action:}
\begin{equation} \label{eq:probing_predictive_distribution}
P\left(
\mathcal{Y}_t
\! \mid \!
\mathcal{H}_t,
\operatorname{do}(\delta\mathbf{u})
\right) \!
= \!
\sum_{k=0}^{2}
\mathcal{L}\!\left(
\mathcal{Y}_t
\! \mid \!
\tau_t=k,
\operatorname{do}(\delta\mathbf{u})
\right)
b_t(k).
\end{equation}
\hl{The intervention notation indicates that the probing action is externally imposed rather than selected by the task controller.}

\hl{For a probing action represented by} $\operatorname{do}(\delta \mathbf{u})$, \hl{we define the expected response under transmission mode} $k$ as
\[
\overline{\mathcal{Y}}_t^{(k)}(\delta \mathbf{u}) := \mathbb{E}\!\left[
\mathcal{Y}_t
\mid
\tau_t=k,\,
\operatorname{do}(\delta \mathbf{u})
\right].
\]
\hl{This quantity is the mean of the corresponding mode-conditioned predictive distribution. 
The probing action has a small magnitude and is approximately orthogonal to the current task direction, increasing separation between mode-conditioned responses while limiting its effect on task progress.}

\hl{In the Decoupled Mode} $(\tau_t=0)$, 
\hl{the proxy exhibits its nominal response to the probing action, while the target remains approximately stationary:}
\[
\overline{\mathcal{Y}}_t^{(0)}(\delta \mathbf{u})
\approx
\begin{bmatrix}
\boldsymbol{\mu}_{\mathrm{n}}(\delta \mathbf{u}) \\
\mathbf{0}
\end{bmatrix}.
\]

\hl{In the Blocked Mode} $(\tau_t=2)$, 
\hl{the environmental constraint attenuates the proxy response, while the target displacement remains negligible:}
\[
\overline{\mathcal{Y}}_t^{(2)}(\delta \mathbf{u})
\approx
\begin{bmatrix}
\lambda_{\mathrm{b}}
\boldsymbol{\mu}_{\mathrm{n}}(\delta \mathbf{u}) \\
\mathbf{0}
\end{bmatrix},
\qquad
0<\lambda_{\mathrm{b}}<1.
\]

\hl{The expected response in the Coupled Mode is obtained from the GP-based predictive distribution. Although target responses in the Decoupled and Blocked modes are both negligible, their proxy responses differ. The mode-conditioned likelihoods in}
Section~\ref{sec:causal_inference} use this difference to update the mode belief and reduce observational ambiguity.

\subsection{Information-Guided Probing}
Building on the interventional formulation above, the robot must actively select probing actions that maximize the information gained about the latent transmission mode.
We formulate this active sensing problem as an optimal experimental design task: selecting a probing action $\delta \mathbf{u}^*$ that maximizes the expected information gain regarding the latent transmission mode $\tau_t$.

We quantify the \ac{eig} using the mutual information between $\tau_t$ and the observation $\mathcal{Y}_t$ \hl{under a candidate probing action, represented in the} \ac{scm} as
$\operatorname{do}(\delta \mathbf{u})$:
\begin{equation}
\label{eq:eig_action}
\delta \mathbf{u}^{*}
=
\argmax_{\delta \mathbf{u} \in \mathcal{U}_{\mathrm{probe}}}
I(\tau_t; \mathcal{Y}_t \mid \mathcal{H}_t, \operatorname{do}(\delta \mathbf{u})),
\end{equation}
where
\(
\mathcal{U}_{\mathrm{probe}}
=
\left\{
\delta \mathbf{u}
\mid
\| \delta \mathbf{u} \|_2
\le
\epsilon_{\mathrm{probe}}
\right\}
\)
\hl{denotes the magnitude-bounded probing-action set, and} $\epsilon_{\mathrm{probe}}=$\SI{5}{\milli\meter}
\hl{is the maximum probing magnitude. 
Candidate directions are sampled approximately orthogonal to the current task direction.}

By definition, mutual information corresponds to the expected reduction in belief entropy,
\[
I(\tau_t; \mathcal{Y}_t \mid \mathcal{H}_t, \operatorname{do}(\delta \mathbf{u}))
=
H(\tau_t \mid \mathcal{H}_t)
-
\mathbb{E}_{\mathcal{Y}_t}
\left[
H(\tau_t \mid \mathcal{Y}_t, \delta \mathbf{u})
\right].
\]
Since the current belief entropy $H(\tau_t \mid \mathcal{H}_t)$ is independent of the probing action, maximizing the expected information gain is equivalent to minimizing the expected posterior entropy.

Directly evaluating the posterior entropy is computationally expensive because it requires integration over the continuous observation space $\mathcal{Y}_t$.
For real-time optimization, we exploit mutual-information symmetry and rewrite the objective using the differential entropy of the predictive distribution in \eqref{eq:probing_predictive_distribution}:
\begin{multline*}
I\left(
\tau_t;\mathcal{Y}_t
\mid
\mathcal{H}_t,\operatorname{do}(\delta\mathbf{u})
\right)
\\[-2pt]
= \!
H\!\left(
\mathcal{Y}_t
\! \mid \!
\mathcal{H}_t,\operatorname{do}(\delta\mathbf{u})
\right)
- \!
\sum_{k=0}^{2}
b_t(k)
H\!\left(
\mathcal{Y}_t
\! \mid \!
\tau_t \!= \!k,\mathcal{H}_t,\operatorname{do}(\delta\mathbf{u})
\right).
\end{multline*}
The second term depends primarily on sensor noise and model uncertainty and varies weakly with the probing action.
As a result, maximizing expected information gain is well approximated by maximizing the entropy of the marginal predictive distribution $H(\mathcal{Y}_t \mid \delta \mathbf{u})$.

\hl{This reformulation has a direct physical interpretation.
The marginal predictive distribution is a Gaussian mixture whose components correspond to the three transmission modes, with means given by the mode-conditioned expected responses}
$\overline{\mathcal{Y}}_t^{(k)}(\delta\mathbf{u})$.
\hl{For real-time evaluation, we approximate this mixture by a moment-matched Gaussian.
Its covariance is}
\[
\bar{\boldsymbol{\Sigma}}
=
\sum_{k=0}^{2} b_t(k)
\left[
\boldsymbol{\Sigma}_k
+
(\boldsymbol{\mu}_k-\bar{\boldsymbol{\mu}})
(\boldsymbol{\mu}_k-\bar{\boldsymbol{\mu}})^\top
\right],
\]
where $\bar{\boldsymbol{\mu}}=\sum_{k=0}^{2}b_t(k)\boldsymbol{\mu}_k$.
\hl{The marginal predictive entropy is then evaluated in closed form from} $\bar{\boldsymbol{\Sigma}}$.
\hl{Greater separation between the mode-conditioned responses generally increases the matched covariance, favouring probing actions that produce more distinguishable responses.}

The probing action space $\mathcal{U}_{\mathrm{probe}}$ is magnitude-bounded, and the resulting optimization problem is non-convex.
We therefore adopt a sample-based strategy.
At each probing step, $N_\mathrm{c}=12$ candidate directions are sampled according to the above criterion and scaled to the probing magnitude $\epsilon_\mathrm{probe}$.
The marginal predictive entropy is evaluated for each candidate using the mode-conditioned models in \eqref{eq:L=0}--\eqref{eq:L=2}, and the most informative probing action is selected,
\begin{equation} \label{eq:sample_based_approx}
\delta \mathbf{u}^*
\approx
\argmax_{i=1,\dots, N_\mathrm{c}}
H(\mathcal{Y}_t \mid \delta \mathbf{u}^{(i)}).
\end{equation}
This sample-based strategy selects informative probing directions without requiring gradient-based optimization.

\section{\acl{dtc}}\label{sec:dtc}
\hl{After active sensing updates the mode belief, the controller combines this belief with the learned} \ac{gp} \hl{coupling model to drive the target toward the desired configuration} $\mathbf{p}_{\mathrm d}$.
Conventional position or impedance controllers are inadequate because they assume persistent target controllability, which fails in the Decoupled and Blocked Modes.
We therefore adopt a \hl{belief-aware} \ac{mpc} that incorporates the inferred transmission mode into state prediction and optimization.

\subsection{Belief-Based State Prediction}

The primary challenge in solving \eqref{eq:objective} is predicting the future target state $\hat{\mathbf{p}}_{t+k}$ under uncertain transmission mode and coupling dynamics.
\hl{For real-time MPC, we approximate the expected-cost formulation in} \eqref{eq:objective} \hl{using the first moment of the uncertain transmission model. 
The current mode belief represents transmission-mode uncertainty, while the local GP posterior mean approximates coupling uncertainty:}
\begin{equation}
\label{eq:linear_predictor}
\hat{\mathbf{p}}_{t+k+1}
\approx
\hat{\mathbf{p}}_{t+k}
+
b_t(\tau=1)
\boldsymbol{\mu}_t(\mathbf{g}_t)
\Delta\mathbf{g}_{t+k},
\end{equation}
where
$\Delta\mathbf{g}_{t+k}
=
\hat{\mathbf{g}}_{t+k+1}-\hat{\mathbf{g}}_{t+k}$.
The local \ac{gp} mean $\boldsymbol{\mu}_t(\mathbf{g}_t)$ is fixed over each prediction horizon and updated in the next cycle.
The \ac{gp} predictive covariance is not propagated through the MPC horizon, but enters the Coupled-mode likelihood in \eqref{eq:L=1}, thereby affecting subsequent mode-belief updates.

When the belief concentrates on the Coupled Mode,
$b_t(\tau=1)\rightarrow1$,
the effective transmission model approaches the learned local \ac{gp} mean.
As the Coupled-mode belief decreases, the predicted transmission is attenuated, reducing the estimated control authority over the target.

\subsection{Quadratic Programming Formulation and Solving}
\acreset{qp}

By substituting the belief-weighted local linear prediction model in \eqref{eq:linear_predictor} into the cost function defined in \eqref{eq:objective}, the finite-horizon control problem can be written in a standard \ac{qp} form at each time step.
\hl{The robot-to-proxy deformation is locally approximated as}
$\Delta\mathbf{g}_t \approx \widehat{\mathbf{J}}_{g,t}\mathbf{u}_t$,
where $\widehat{\mathbf{J}}_{g,t}$ \hl{is the proxy deformation Jacobian
estimated using the shape-servoing method in} \cite{navarro2013model}.
\hl{Within each MPC optimization,}
$b_t(\tau=1)$, $\boldsymbol{\mu}_t(\mathbf{g}_t)$, and
$\widehat{\mathbf{J}}_{g,t}$ \hl{are held fixed over the prediction horizon
and updated at the next control cycle.}
Specifically, by stacking the predicted target dynamics over the
prediction horizon $T$, the optimization problem becomes
\begin{equation}\label{eq:qp}
\mathbf{u}^*_{t:t+T-1}
=
\argmin_{\mathbf{u}}
\frac{1}{2}\mathbf{u}^{\top}\mathbf{H}\mathbf{u}
+
\mathbf{f}^{\top}\mathbf{u},
\end{equation}
where
$\mathbf{u}
=
[\mathbf{u}_t^\top,\dots,\mathbf{u}_{t+T-1}^\top]^\top$
denotes the stacked control sequence.
The Hessian matrix $\mathbf{H}$ and gradient vector $\mathbf{f}$ are
obtained by expanding the quadratic cost over the prediction horizon as
\[
\mathbf{H}
=
2\left(
\mathbf{\Gamma}^\top
\bar{\mathbf{Q}}
\mathbf{\Gamma}
+
\bar{\mathbf{R}}
\right),
\qquad
\mathbf{f}
=
2\mathbf{\Gamma}^\top
\bar{\mathbf{Q}}
\left(
\mathbf{\Phi}\mathbf{p}_t
-
\mathbf{P}_{\rm d}
\right),
\]
where $\bar{\mathbf{Q}}$ and $\bar{\mathbf{R}}$ are the block-diagonal
extensions of the stage cost matrices $\mathbf{Q}$ and $\mathbf{R}$
over the horizon.
The matrix $\mathbf{\Gamma}$ is the convolution matrix constructed
from the \hl{fixed belief-weighted local control Jacobian}
$\mathbf{M}_t=b_t(\tau=1)\boldsymbol{\mu}_t(\mathbf{g}_t)\widehat{\mathbf{J}}_{g,t}$,
while $\mathbf{\Phi}$ propagates the current target state $\mathbf{p}_t$ through the prediction horizon under zero input.
The vector $\mathbf{P}_{\rm d}$ stacks the desired target configuration $\mathbf{p}_{\rm d}$ across the horizon.

The resulting \ac{qp} \hl{implements the belief-aware deformation controller used in this work; it does not impose explicit force, tissue-strain, or collision-avoidance constraints and is solved using the standard numerical solver OSQP.}

\subsection{Algorithm Summary and Implementation}
The complete control architecture integrates mode inference with the belief-aware deformation controller.
Mode inference and probing operate at a higher level, whereas the \ac{mpc} computes task-directed actions each control cycle.
The first optimized action, $\mathbf{u}_t^{*}$, is implemented as an incremental Cartesian position command executed by a low-level position servo.

\hl{For task monitoring, the target shape error is quantified by the mean Euclidean distance between the current and desired target feature positions:}
\[e_t =
\frac{1}{N_f}
\sum_{i=1}^{N_f}
\left\|
\mathbf{p}_{i,t}-\mathbf{p}_{i,d}
\right\|_2 ,
\]
where $N_f$ is the number of target features and $\mathbf{p}_{i,t}$ and $\mathbf{p}_{i,d}$ denote the current and desired positions of the $i$-th target feature, respectively.

\hl{A reduction in target shape error alone is insufficient to determine task completion, as perception noise, online model updates, and viscoelastic tissue relaxation may induce residual corrections.
We therefore distinguish target-shape convergence from robot-motion settlement using the joint condition}
\begin{equation}
\label{eq:settlement_condition}
e_t \leq \epsilon_{\rm p},
\qquad
\left\|\mathbf{u}_t^{*}\right\|_2
\leq \epsilon_{\rm u}.
\end{equation}
Here, $\epsilon_{\rm p}$ and $\epsilon_{\rm u}$ \hl{denote the target-shape and Cartesian-command settlement thresholds, respectively.
The manipulation is considered settled only when}
\eqref{eq:settlement_condition} holds for $N_{\rm s}$ consecutive control cycles, where $N_{\rm s}$ defines the settlement window.
In all experiments, we used $\epsilon_{\rm p}=5$\si{\milli\meter},
$\epsilon_{\rm u}=0.6$\si{\milli\meter}, and $N_{\rm s}=20$.
\hl{The controller then holds the current end-effector pose.}

The overall procedure is summarized in Algorithm~1.

\begin{algorithm}[!htb]
\caption{\hl{Causal-Aware Indirect Manipulation}}
\begin{algorithmic}[1]

\State \textbf{Initialize:}
Mode belief $b_0(\tau)=\left[\frac{1}{3},\frac{1}{3},\frac{1}{3}\right]$, 
\ac{gp} model $\mathcal{GP}_0$, horizon $T$, response buffer 
$\mathcal{B}\leftarrow\emptyset$, nominal-model flag 
$q_{\mathrm n}\leftarrow0$, GP confidence counter 
$n_{\mathrm{GP}}\leftarrow0$, and settlement counter 
$n_{\mathrm s}\leftarrow0$.

\State \textbf{Loop} for $t=0,1,2,\dots$

\State \quad \textbf{If } $q_{\mathrm n}=0$\textbf{:}

\State \qquad Select a bounded probing action
$\mathbf{u}_t\in\mathcal{U}_{\mathrm{probe}}$.

\State \quad \textbf{Otherwise:}

\State \qquad \textbf{If }
$H(b_t)>H_{\mathrm{th}}$\textbf{:}

\State \qquad\quad
$\delta\mathbf{u}^{*}
\leftarrow
\arg\max_{\delta\mathbf{u}\in\mathcal{U}_{\mathrm{probe}}}
I(\tau;\mathcal{Y}\mid\delta\mathbf{u})$
using \eqref{eq:eig_action} and
\eqref{eq:sample_based_approx},
and set $\mathbf{u}_t\leftarrow\delta\mathbf{u}^{*}$.

\State \qquad \textbf{Else:}

\State \qquad\quad Obtain
$\mathbf{u}^{*}_{t:t+T-1}$ by solving \eqref{eq:qp}.

\State \qquad\quad
$n_{\mathrm s}\leftarrow n_{\mathrm s}+1$
if
$e_t\leq\epsilon_{\mathrm p}$
and
$\|\mathbf{u}^{*}_t\|\leq\epsilon_{\mathrm u}$;
otherwise
$n_{\mathrm s}\leftarrow0$.

\State \qquad\quad \textbf{If }
$n_{\mathrm s}\geq N_{\mathrm s}$\textbf{:}
hold current robot pose and \textbf{break}.

\State \qquad\quad
Set $\mathbf{u}_t\leftarrow\mathbf{u}^{*}_t$.

\State \quad Apply $\mathbf{u}_t$ and observe
$\mathcal{Y}_t=
[\Delta\mathbf{g}_t^\top,
\Delta\mathbf{p}_t^\top]^\top$.

\State \quad \textbf{If } $q_{\mathrm n}=0$\textbf{:}

\State \qquad Store
$(\mathbf{u}_t,\mathcal{Y}_t)$ in $\mathcal{B}$.

\State \qquad \textbf{If } $|\mathcal{B}|<N_0$\textbf{:}
\textbf{continue}.

\State \qquad Initialize $P_{\mathrm n}$ from $\mathcal{B}$,
construct $P_{\mathrm{stiff}}$ using $\lambda_{\mathrm b}$,
and replay the buffered likelihoods to initialize $b_{t+1}(\tau)$.

\State \qquad Initialize $\mathcal{GP}_{t+1}$ using buffered samples admitted by Coupled-mode confidence gate
$(\beta_{\mathrm{GP}},N_{\mathrm{GP}})$,
and $q_{\mathrm n}\leftarrow1$.

\State \quad \textbf{Otherwise:}

\State \qquad Update $b_{t+1}(\tau)$ using Bayes' rule and
\eqref{eq:L=0}--\eqref{eq:L=2}.

\State \qquad
$n_{\mathrm{GP}}\leftarrow n_{\mathrm{GP}}+1$
if $b_{t+1}(\tau=1)>\beta_{\mathrm{GP}}$;
otherwise $n_{\mathrm{GP}}\leftarrow0$.

\State \qquad \textbf{If }
$n_{\mathrm{GP}}\geq N_{\mathrm{GP}}$\textbf{:}
update \ac{gp} coupling model via \eqref{eq:gp_posterior}.

\end{algorithmic}
\end{algorithm}

\section{Simulation Experiments}

All simulations were conducted using the SOFA framework \cite{faure:hal-00681539}. 
The virtual environment includes deformable models of the stomach, pancreas, liver, and surrounding tissues, together with standard surgical instruments, to mimic pancreatic exposure through stomach manipulation. 
Additional rigid blocks are introduced as external obstacles to generate collision constraints.
Adhesion between the stomach and pancreas is synthetically generated using predefined adhesion nodes and a sticking force activated upon contact. 
Because the simulator provides full access to ground-truth states, organ positions, adhesion configurations, and collisions are fully observable.

\begin{figure*}[!hbt]
    \centering
    \addtolength{\abovecaptionskip}{-10pt}
    \input{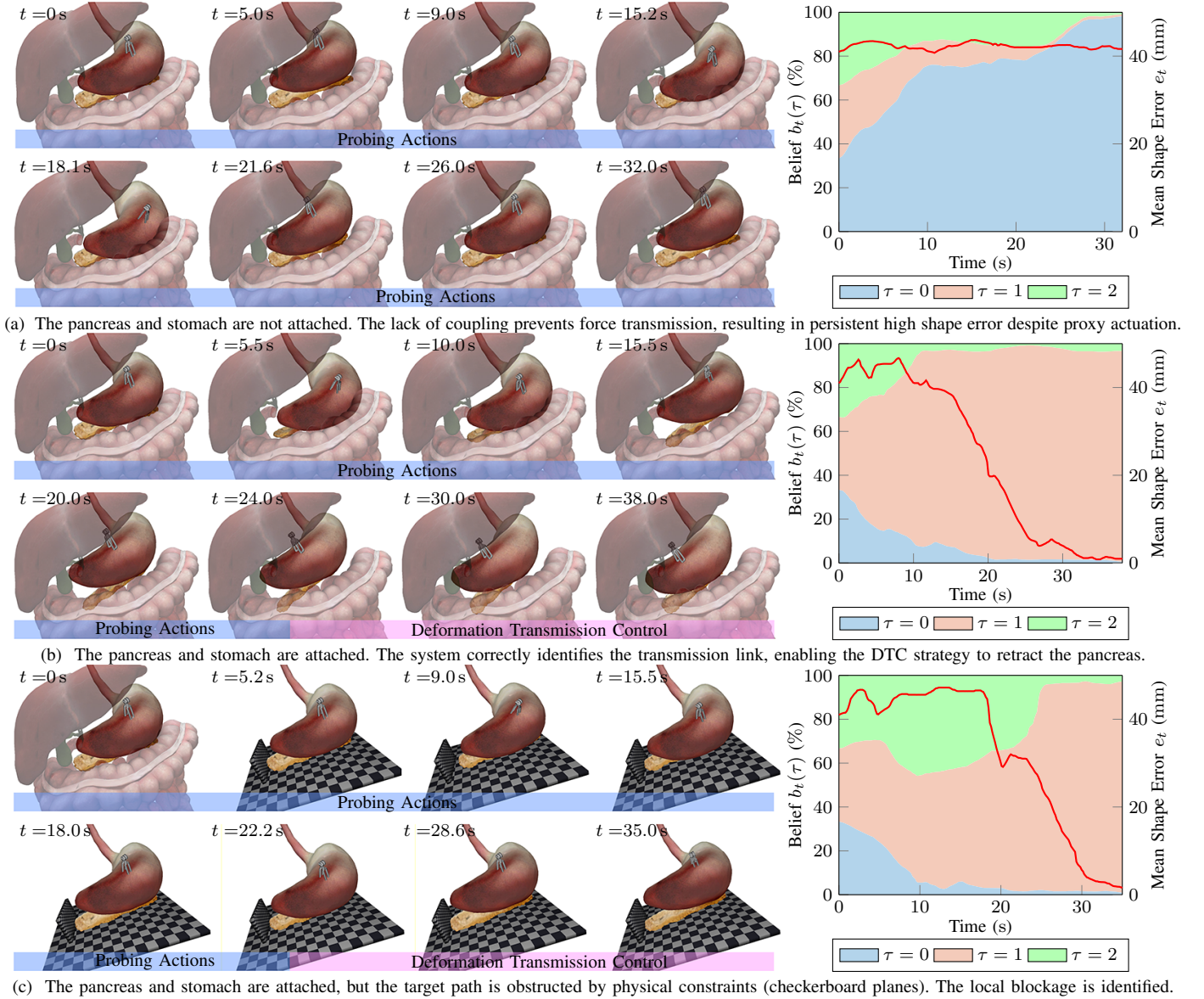}
    \caption{\hl{Representative simulation results for the three transmission modes: (a) Decoupled, (b) Coupled, and (c) Blocked.
    The image sequences show stomach-proxy manipulation and the pancreatic-target response.
    The stacked area plots show the mode belief} $b_t(\tau)$, \hl{while the red curve shows the mean target-shape error.
    Target deformation and error reduction are evident in the Coupled Mode, whereas target motion remains limited in the Decoupled and Blocked modes.}}
    \label{fig:simulation}
\end{figure*}

\hl{For all simulation trials, the first} $N_0=10$ \hl{probing responses were buffered to initialize the local nominal proxy-response distribution} $P_{\mathrm{n}}$.
\hl{For the Blocked-mode trials, the initial response window was completed before the target reached the synthesized external constraint.
The blocked-response distribution} $P_{\mathrm{stiff}}$ \hl{was then constructed from} $P_{\mathrm{n}}$ using a response-attenuation factor
$\lambda_{\mathrm{b}}=0.4$.
\hl{Both distributions were kept fixed for the remainder of each trial.
The buffered observations were subsequently processed by the Bayesian filter to initialize the mode belief; thus, no separate calibration trajectory was added.}

To quantify the strength of deformation transmission, we define the \ac{dtr} $\alpha_t$ as the ratio between the mean target and proxy feature displacements:
\[
\alpha_t =
\frac{
\frac{1}{N_\mathrm{p}}\sum_{i=1}^{N_\mathrm{p}}
\left\|\Delta\mathbf{p}_{i,t}\right\|_2
}{
\frac{1}{N_\mathrm{g}}\sum_{j=1}^{N_\mathrm{g}}
\left\|\Delta\mathbf{g}_{j,t}\right\|_2+\epsilon_{\mathrm{DTR}}
},
\]
where $N_\mathrm{p}$ and $N_\mathrm{g}$ denote the numbers of target and proxy features, respectively, and $\epsilon_{\mathrm{DTR}}$ is a small positive constant for numerical stability.
A \ac{dtr} close to zero indicates little target motion relative to proxy deformation, as may occur in the Decoupled or Blocked Mode, whereas a larger \ac{dtr} indicates stronger deformation transmission.
In simulation, the ground-truth transmission mode is determined directly from the known adhesion and collision conditions, while the \ac{dtr} is used to quantify the strength of deformation transmission.

\subsection{Setup and Latent Mode Synthesis}
To evaluate the system’s behaviour under structural uncertainty, we synthesized three transmission modes corresponding to different physical interaction regimes.

The Decoupled Mode ($\tau = 0$) was simulated by disabling predefined adhesion nodes between the stomach and pancreas. In this setting, only a transient sticking force may be triggered upon contact detection, resulting in negligible deformation transmission and effectively modelling a loss of coupling.
In Coupled Mode ($\tau\! =\! 1$) , predefined adhesion nodes between the stomach and pancreas remained active, allowing proxy deformation to be consistently transmitted to the target organ through viscoelastic coupling.
The Blocked Mode ($\tau = 2$) \hl{was instantiated by introducing a rigid collision object along the expected deformation path of the pancreas, while preserving the adhesive coupling used in the Coupled Mode. 
In these trials, the Blocked Mode is active when the target engages the external constraint and need not persist throughout the entire trial.}
This configuration constrains target motion through external contact, despite ongoing proxy deformation.

As shown in Fig. \ref{fig:simulation}, we conducted representative simulation trials to characterize the system's temporal response under each transmission mode.
The figure visualizes the evolution of the kinematic observables, specifically the proxy deformation $\Delta \mathbf{g}$ and target motion $\Delta \mathbf{p}$, alongside the real-time inference of the posterior belief $b_t(\tau)$.
In the Decoupled and Blocked scenarios, while the target displacements are similarly negligible ($\Delta \mathbf{p} \approx \mathbf{0}$), the active sensing strategy elicits distinct signatures in the proxy's deformation.
When the latent transmission mode $\tau=1$ is detected, proxy deformation is transmitted to the target, enabling the controller to drive the target toward the desired shape.

\subsection{\hl{Ablation of Probing Actions}}
\label{sec:mode_identification}

\hl{To isolate the contribution of ambiguity-triggered probing, we compared CADeT with a no-additional-probing variant using the same initialization, Bayesian filter in} \eqref{eq:bayes}, \hl{and simulation settings.
The baseline relied only on task-directed control responses to update the mode belief, whereas CADeT additionally applied probing actions when the belief remained ambiguous.
All other inference components were kept unchanged.}
We conducted $N=150$ randomized trials (\num{50} per mode), with the ground-truth transmission mode withheld from the inference module.
\hl{A Blocked trial was considered correctly identified if evidence for the Blocked Mode was observed during constraint engagement; the mode need not remain Blocked thereafter.}

\subsubsection{\hl{Classification Accuracy}}
Fig.~\ref{fig:confusion_matrix} \hl{presents the trial-level confusion matrices,
with each row normalized over 50 trials for the corresponding interaction condition.}
The \hl{no-probing variant} misclassified \SI{92}{\percent} of the Blocked Mode trials ($\tau=2$) as the Decoupled Mode ($\tau=0$).
This result is consistent with Proposition~\ref{prop:indistinguishability}: when no additional probing action is applied and $\Delta \mathbf{p} \approx \mathbf{0}$, \hl{task-directed observations may be insufficient to distinguish the Decoupled and Blocked modes.}
In contrast, \hl{the full CADeT framework achieved an overall trial-level identification accuracy of} \SI{99.3}{\percent} (\num{149}/\num{150} trials).
\hl{The selected probing actions elicited distinguishable proxy-deformation responses in the Decoupled and Blocked modes, allowing the Bayesian filter to separate the two hypotheses and reduce observational ambiguity.}

\begin{figure}[!hbt]
    \centering
    \addtolength{\abovecaptionskip}{-18pt}
    \input{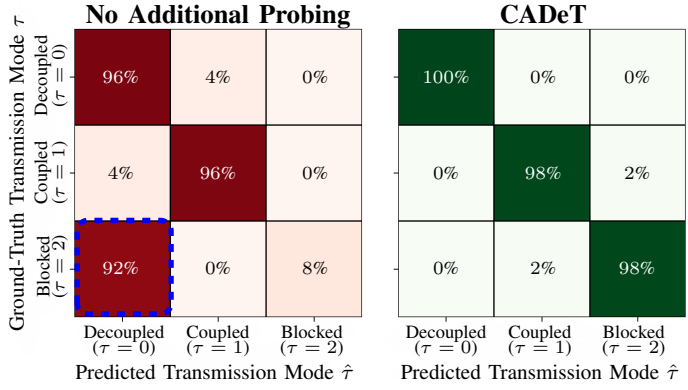}
    \caption{\hl{Confusion matrices comparing mode-identification accuracy between the no-additional-probing and the full CADeT framework across the Decoupled, Coupled, and Blocked modes.}}
    \label{fig:confusion_matrix}
\end{figure}

\subsubsection{\hl{Belief Convergence}}
\hl{We further evaluated the convergence of the mode belief by tracking its Shannon entropy over the simulation horizon:}
\[
H(b_t) = -\sum_{k=0}^{2} b_t(k)\ln b_t(k).
\]
As shown in Fig.~\ref{fig:entropy}, CADeT reduced the mode-belief entropy across all three transmission modes, eventually falling below \SI{0.25}{nat}.
By contrast, the \hl{no-probing variant} remained uncertain in the Decoupled and Blocked modes, with entropy near \SI{0.65}{nat} because \hl{task-directed observations were insufficient to reliably distinguish these two modes.}
\hl{Together with the classification results, this ablation shows that probing improves both mode-identification accuracy and belief convergence under ambiguous conditions.}

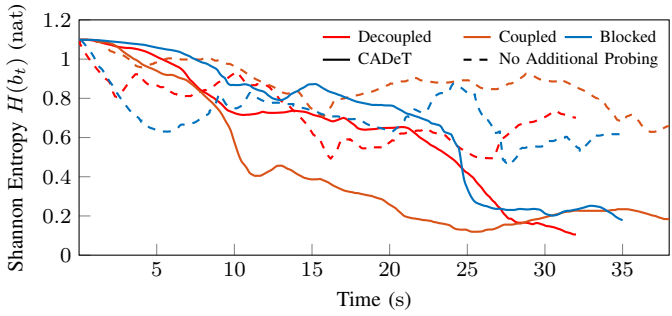
\begin{figure}[!hbt]
    \centering
    \addtolength{\abovecaptionskip}{-20pt}
    \begin{tikzpicture}
    \footnotesize
    \begin{axis}[
        name=position2,
        xlabel={Time (\unit{\second})}, % Enclosed in braces for safety
        ylabel={Shannon Entropy $H(b_t)$ (\unit{nat})}, % Removed negative spacing for clarity
        xmin=0, xmax=38,
        ymin=0, ymax=1.2,
        xtick={5,10,...,35},
        ytick={0,0.2,...,1.2},
        xtick pos=left,
        ytick pos=left,
        width=1.06\linewidth,
        height=0.53\linewidth,
        % --- LEGEND STYLE ---
        legend style={
            at={(1,1)}, 
            anchor=north east, 
            nodes={scale=0.8, transform shape},
            legend columns=3,
            legend cell align=left,
            legend image post style={xscale=0.6},
            column sep=1pt, 
            /tikz/column 2/.style={column sep=10pt}, % Space between cols
            draw=none, % Optional: removes the box border for a cleaner look
            fill=none  % Optional: removes background fill
        }
    ]

    % --- ACTUAL DATA (Hidden from legend via 'forget plot') ---
    % Mode 0 (Red)
    \addplot[smooth, color=red, thick, forget plot] table [x index=0, y index=1] {data/simulation_tau0_entropy.txt};
    \addplot[dashed, smooth, color=red, thick, forget plot] table [x index=0, y index=1] {data/simulation_tau0_entropy_passive.txt};

    % Mode 1 (Orange)
    \addplot[smooth, color=techorange, thick, forget plot] table [x index=0, y index=1] {data/simulation_tau1_entropy.txt};
    \addplot[dashed, smooth, color=techorange, thick, forget plot] table [x index=0, y index=1] {data/simulation_tau1_entropy_passive.txt};

    % Mode 2 (Blue)
    \addplot[smooth, color=techblue, thick, forget plot] table [x index=0, y index=1] {data/simulation_tau2_entropy.txt};
    \addplot[dashed, smooth, color=techblue, thick, forget plot] table [x index=0, y index=1] {data/simulation_tau2_entropy_passive.txt};

    % --- CUSTOM LEGEND ENTRIES ---
    % Column 1: The Modes (Colors)
    \addlegendimage{no markers, red, thick}
    \addlegendentry{Decoupled}
    
    \addlegendimage{no markers, techorange, thick}
    \addlegendentry{Coupled}

    \addlegendimage{no markers, techblue, thick}
    \addlegendentry{Blocked}

    \addlegendimage{no markers, black, thick} % Generic solid line
    \addlegendentry{CADeT}
    
    \addlegendimage{no markers, black, dashed, thick} % Generic dashed line
    \addlegendentry[overlay]{No Additional Probing}

    \end{axis}
\end{tikzpicture}
    \caption{\hl{Evolution of the mode-belief entropy for the no-probing variant and the full CADeT framework across the three transmission modes. Lower entropy indicates higher confidence in the inferred mode.}}
    \label{fig:entropy}
\end{figure}

\subsection{\hl{Ablation of Probing Policy}}
\label{sec:probing_policy}

\subsubsection{\hl{Quantitative Comparison}}
\hl{To isolate the contribution of the information-guided probing policy, we compared the full CADeT framework with a random-probing variant.
Both variants used the same mode-conditioned likelihoods, Bayesian belief update, GP coupling model, probing trigger, and feasible probing set} $\mathcal{U}_{\mathrm{probe}}$.
\hl{The only difference was the selection of the probing action: CADeT selected} $\delta\mathbf{u}^{*}$ using the information-gain criterion in \eqref{eq:sample_based_approx}, whereas the random-probing variant selected a feasible probing action uniformly at random.

\hl{We evaluated both variants using the same} $N=150$ \hl{randomized trials} (\num{50} \hl{per transmission mode) and identical initial conditions.
The evaluation considered overall mode-identification accuracy, Blocked-mode recall, belief-convergence time, and probe count. 
Belief convergence was defined as the first time at which}
$H(b_t)<0.25$~nat.
\hl{The latter two metrics are reported as mean $\pm$ standard deviation, with probe count excluding the initial} $N_0=10$ \hl{responses.}

\begin{table}[!hbt]
    \centering
    \caption{\hl{Ablation of the probing policy.}}
    \label{tab:probing_ablation}
    \setlength{\tabcolsep}{4pt}
    \begin{tabular}{lcccc}
    \hline
    \textbf{Method}
    & \makecell{Mode ID.\\Accuracy (\%)}
    & \makecell{Blocked\\Recall (\%)}
    & \makecell{Belief Conv.\\Time (s)}
    & \makecell{Probe\\Count} \\
    \hline
    Random probing
    & $94.0$
    & $90.0$
    & $14.9 \pm 1.5$
    & $15.5 \pm 4.1$ \\
    CADeT
    & $\mathbf{99.3}$
    & $\mathbf{98.0}$
    & $\boldsymbol{13.2 \pm 0.4}$
    & $\boldsymbol{6.7 \pm 1.1}$ \\
    \hline
    \end{tabular}
\end{table}

\hl{Table}~\ref{tab:probing_ablation} \hl{shows that information-guided probing increased the overall mode-identification accuracy from} \SI{94.0}{\percent} \hl{to} \SI{99.3}{\percent} \hl{while reducing the number of additional probing actions from} $15.5 \pm 4.1$ \hl{to} $6.7 \pm 1.1$.
\hl{Thus, informative action selection improves mode identification with fewer probing interactions.}

\subsubsection{\hl{Random-Probing Variability}}

\hl{To examine the variability of random probing, we analysed} \num{50}
\hl{trials in a fixed Coupled Mode scenario.}
The probing actions $\delta\mathbf{u}$ were uniformly sampled from
\hl{the feasible probing set under the same magnitude constraints.}
Fig.~\ref{fig:trajectories} \hl{shows representative mode-belief trajectories on the 2-simplex.}

\begin{figure*}[!hbt]
    \centering
    \addtolength{\abovecaptionskip}{-20pt}
    \input{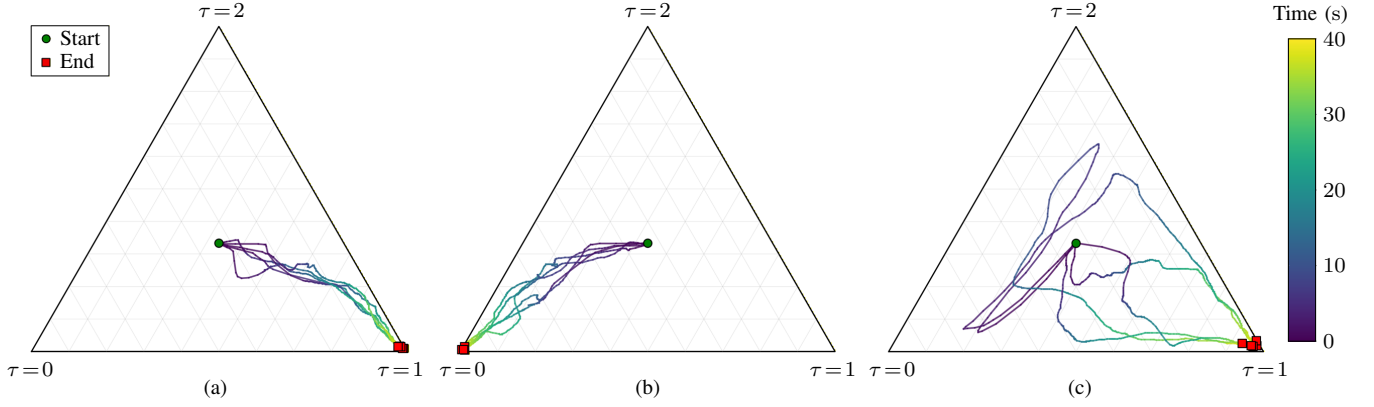}
    \caption{Representative mode-belief trajectories on the 2-simplex under random probing in the Coupled Mode.
    The vertices denote certainty in the Decoupled ($\tau=0$), Coupled ($\tau=1$), and Blocked ($\tau=2$) modes, and the colour gradient indicates temporal progression from the initial uniform belief (green circle) to the final belief (red square).
    (a) and (b) show rapid convergence, whereas (c) shows a slower trajectory with temporary ambiguity.}
    \label{fig:trajectories}
\end{figure*}

\hl{The trajectories demonstrate the direction-dependent variability of random probing.
Distinctive tissue responses led to rapid convergence toward the ground-truth mode}
(Fig.~\ref{fig:trajectories}a, b),
\hl{whereas temporarily ambiguous responses resulted in slower belief updates}
(Fig.~\ref{fig:trajectories}c).
These observations are consistent with Table~\ref{tab:probing_ablation}:
\hl{random probing can provide useful mode information, but the information-guided policy reduces its dependence on probing direction by favouring actions with more distinguishable predicted responses.}

\subsection{Robustness and Sensitivity Analysis}

\subsubsection{Blocked-Response Attenuation}

\hl{We further evaluated the sensitivity of mode inference to the attenuation factor} $\lambda_{\mathrm{b}}$ \hl{used to construct} $P_{\mathrm{stiff}}$ \hl{from the locally initialized nominal response.
The same} \num{150} trials used in Section~\ref{sec:mode_identification} were re-evaluated with
$\lambda_{\mathrm{b}}\in\{0.2,0.4,0.6,0.8,0.95\}$,
\hl{while all other inference parameters and the initial response window} $N_0$ \hl{were kept unchanged.
For each setting, we measured the overall mode-identification accuracy, Blocked-mode recall, Blocked-to-Decoupled confusion rate, and the time for mode-belief entropy to fall below} \SI{0.25}{nat}.

\hl{As summarized in Table}~\ref{tab:lambda_sensitivity},
\hl{the inference performance was most stable for}
$\lambda_{\mathrm{b}}=0.4$--$0.6$.
\hl{We therefore used} $\lambda_{\mathrm{b}}=0.4$ \hl{in the subsequent experiments.
Values closer to one reduced the separation between}
$P_{\mathrm{n}}$ \hl{and} $P_{\mathrm{stiff}}$,
\hl{increasing Blocked-to-Decoupled confusion, whereas the smaller value}
$\lambda_{\mathrm{b}}=0.2$
\hl{required longer to accumulate consistent evidence.}

\begin{table}[!hbt]
\centering
\caption{\hl{Sensitivity of mode inference to the blocked-response attenuation factor} $\lambda_{\mathrm{b}}$.}
\label{tab:lambda_sensitivity}
\setlength{\tabcolsep}{4pt}
\begin{tabular}{ccccc}
\hline
$\boldsymbol{\lambda}_{\mathrm{b}}$
& \makecell{Mode\\Accuracy (\%)}
& \makecell{Blocked\\Recall (\%)}
& \makecell{B$\rightarrow$D\\Rate (\%)}
& \makecell{Belief Conv.\\Time (s)} \\
\hline
$0.2$  & $98.0$  & $94.0$  & $4.0$  & $15.4 \pm 0.5$ \\
$0.4$  & $\boldsymbol{99.3}$ & $\boldsymbol{98.0}$ & $\boldsymbol{0.0}$ & $\boldsymbol{13.2 \pm 0.4}$ \\
$0.6$  & $98.7$ & $96.0$ & $4.0$ & $13.7 \pm 0.3$ \\
$0.8$  & $95.3$  & $86.0$  & $14.0$ & $17.9 \pm 2.1$ \\
$0.95$ & $91.3$  & $82.0$  & $18.0$ & $20.2 \pm 2.4$ \\
\hline
\end{tabular}
\end{table}

\subsubsection{\hl{Robustness to Early Mode Misclassification}}
\hl{We evaluated whether transient mode-inference errors could affect manipulation by corrupting the online GP coupling model.
A fixed Coupled-mode scenario was tested under three conditions: clean reference, direct GP updating, and confidence-gated updating.
For the latter two,} $M=4$ \hl{non-Coupled observations were introduced early and temporarily misclassified as Coupled.
Direct updating admitted them immediately, whereas the same confidence gate was used before updating the GP.
All subsequent observations were ground-truth Coupled.}

\hl{Fig.}~\ref{fig:gp_contamination} \hl{shows representative target-shape error trajectories.
Direct updating caused slower convergence after incorporating the incorrect observations, whereas confidence-gated updating remained close to the clean reference by rejecting them.
These results show that the confidence gate reduces the effect of early mode misclassification on subsequent manipulation.}

\begin{figure}[!hbt]
    \centering
    \addtolength{\abovecaptionskip}{-20pt}
    \begin{tikzpicture}
    \footnotesize
    \begin{axis}[
        xlabel={Time (\unit{\second})},
        ylabel={Mean Shape Error $e_t$ (\unit{\milli\meter})},
        xmin=0, xmax=38,
        ymin=0, ymax=72,
        xtick pos=left, ytick pos=left,
        xtick={5,10,...,35},
        ytick={0,20,...,60},
        width=1.06\linewidth,
        height=0.53\linewidth,
        legend pos=north east,
        legend style={
            nodes={scale=0.8, transform shape},
        }
    ]

    \addplot[
        red,
        thick,
        % dashed
    ] table [x index=0, y index=4] {data/simulation_tau1.txt};

    \addlegendentry{Gated}

    \addplot[
        techblue,
        thick,
        dashed
    ] table [x index=0, y index=1] {data/sim_phamton_tau1_clean.txt};
    \addlegendentry{Clean}

    \addplot[
        techorange,
        thick,
        % dashed
    ] table [x index=0, y index=1] {data/sim_phantom_tau1_directupdate.txt};
    \addlegendentry{Direct update}

    \end{axis}
\end{tikzpicture}
    \caption{\hl{Shape-control robustness to early mode misclassification under different GP update strategies.}}
    \label{fig:gp_contamination}
\end{figure}
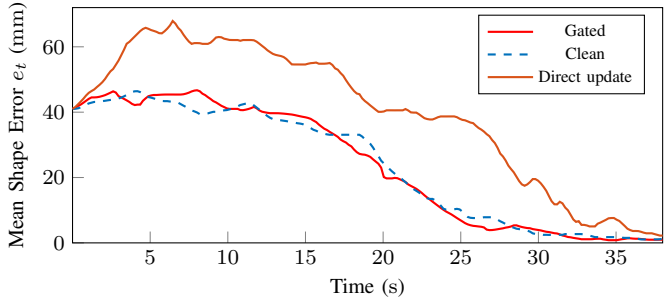

\subsection{\hl{Component Ablation for Manipulation}}
\label{sec:component_ablation}

\hl{The preceding ablations isolate the effects of the additional
probing actions and their information-guided selection.
We next evaluated the contributions of the mode-conditioned
manipulation model, the GP coupling model, and the belief-aware MPC.
All variants were tested under the same coupling conditions, initial
configurations, desired target configurations, and controller limits
used for the full CADeT framework.}

\hl{Three variants were considered.
First, in the \emph{No mode conditioning} variant, the active sensing and mode-inference procedures were retained, but the inferred mode was not used to condition the coupling model or target prediction.
All observations were used for GP updating, and the coupling belief $b_t(\tau=1)$ in the MPC prediction was replaced by unity.
Second, the \emph{Hard-mode MPC} variant retained the same mode inference and GP model but replaced the continuous coupling belief by the binary maximum a posteriori estimate} $\mathbb{I}(\hat{\tau}_t=1)$ where $\hat{\tau}_t=\arg\max_{\tau} b_t(\tau)$.
\hl{Third, in the \emph{Jacobian coupling} variant, the GP coupling model was replaced by the recursive Jacobian estimate used in the Jacobian-based controller, while mode inference, probing, and belief-based modulation were retained.}

\begin{table*}[t]
\centering
\caption{\hl{Component ablation of mode conditioning, belief weighting, and GP-based coupling estimation.}}
\label{tab:component_ablation}
\setlength{\tabcolsep}{7pt}
\begin{tabular}{lcccccc}
\hline
\textbf{Variant}
& Mode Inference
& GP Coupling
& Belief Weighting
& Success Rate
& Task Time (s)
& Average DTR \\
\hline
No mode conditioning
& Yes & Yes & None & $84.0\%$ & $37.7 \pm 4.2$ & $7.54\%$ \\

Hard-mode MPC
& Yes & Yes & Hard & $90.0\%$ & $ 45.1 \pm 5.1$ & $8.67\%$ \\

Jacobian coupling
& Yes & No & Soft & $86.0\%$ & $ 38.6 \pm 3.7$ & $7.82\%$ \\

\textbf{Full CADeT}
& Yes & Yes & Soft
& $\mathbf{98.0\%}$
& $\mathbf{30.1 \pm 2.7}$
& $\mathbf{11.26\%}$ \\
\hline
\end{tabular}
\end{table*}

\hl{Table}~\ref{tab:component_ablation} \hl{summarizes the manipulation performance of the component ablations.
The No mode conditioning, Hard-mode MPC, and Jacobian coupling variants achieved success rates of}
\SI{84.0}{\percent}, \SI{90.0}{\percent}, \hl{and} \SI{86.0}{\percent},
\hl{with task times of}
$37.7 \pm 4.2$~\si{\second}, $45.1 \pm 5.1$~\si{\second}, \hl{and}
$38.6 \pm 3.7$~\si{\second},
\hl{respectively.
Full CADeT achieved the highest success rate of} \SI{98.0}{\percent}
\hl{and average DTR of} \SI{11.26}{\percent},
\hl{with the shortest task time of} $30.1 \pm 2.7$~\si{\second}.

\subsection{\hl{Comparison with Existing Methods}}
\label{sec:baseline_comparison}

\hl{We compared CADeT with four representative
deformable-object manipulation methods: the Jacobian-based controller}
\cite{navarro2013model}, GP-WRM \cite{hu2023occlusion}, online GPR \cite{hu2018three}, and the FEM-based controller \cite{saghour2025dual}.
\hl{All methods were evaluated on identical simulated indirect manipulation tasks and coupling conditions.
The baselines did not use explicit transmission-mode inference or mode-dependent switching, whereas CADeT continuously updated the mode belief during manipulation.}

\hl{The comparison evaluates manipulation performance under different
physical coupling strengths.
In simulation, the coupling between the proxy and target was varied by
adjusting the density of connection nodes at the tissue interface.}

Fig.~\ref{fig:dtr} shows the temporal evolution of the \ac{dtr}, while Table~\ref{tab:comparison} summarizes the aggregate manipulation performance.
\hl{CADeT achieved the highest success rate, normalized convergence-speed factor, and average} \ac{dtr} among the evaluated methods.
\hl{Because this comparison evaluates the complete frameworks, the contributions of probing, probing-action selection, and the manipulation components are isolated separately in Sections}~\ref{sec:mode_identification},
\ref{sec:probing_policy}, \hl{and the component ablation, respectively.}

\begin{figure}[!hbt]
    \centering
    \addtolength{\abovecaptionskip}{-20pt}
    \begin{tikzpicture}
    \footnotesize
    \begin{axis}[
        name=position2,
        xlabel=Time (\unit{\second}),
        ylabel=DTR $ \alpha_t$ (\%),
        xmin=0, xmax=10,
        ymin=0, ymax=24,
        xtick={0,2,...,10},
        ytick={0,6,12,...,24},
        xtick pos=left,
        ytick pos=left,
        width=1.05\linewidth,
        height=0.525\linewidth,
        legend style={
            nodes={scale=0.8, transform shape},
            legend columns=3
        }
    ]

    \addplot[smooth,color=red, thick]  table  [x index=0, y index=1] {data/CADeT_DTR.txt};
    \addlegendentry{CADeT}
    \addplot[smooth,color=techorange, thick]  table  [x index=0, y index=1] {data/GPWRM_DTR.txt};
    \addlegendentry{GP-WRM}

    \addplot[smooth,color=techblue, thick]  table  [x index=0, y index=1] {data/GPR_DTR.txt};
    \addlegendentry{Online GPR}

    \addplot[smooth,color=blue, thick]  table  [x index=0, y index=1] {data/Jacobian_DTR.txt};
    \addlegendentry{Jacobian}

    \addplot[smooth,color=pink, thick]  table  [x index=0, y index=1] {data/FEM_DTR.txt};
    \addlegendentry{FEM}
    
    \end{axis}

\end{tikzpicture}
    \caption{\hl{Deformation transmission ratio achieved by CADeT and
    the baseline controllers. A higher ratio indicates more effective
    transmission of proxy deformation to the target.}}
    \label{fig:dtr}
\end{figure}
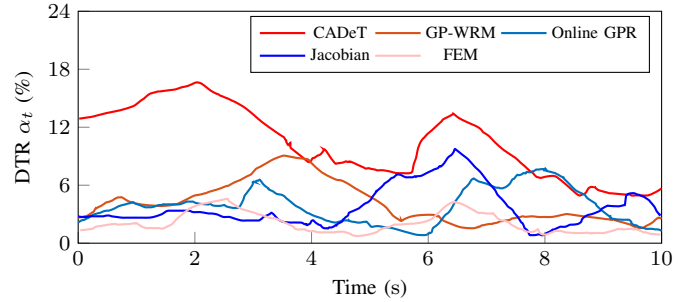

\begin{table}[!hbt]
\centering
\caption{\hl{Comparison with existing manipulation methods under the
tested coupling conditions.}}
\label{tab:comparison}
\begin{tabular}{lccc}
\hline
\textbf{Method}
& Success Rate
& Conv. Speed
& Average DTR \\
\hline
Jacobian-based \cite{navarro2013model}
& $42.0\%$ & $1.0\times$ & $4.25\%$ \\

GP-WRM \cite{hu2023occlusion}
& $68.0\%$ & $1.17\times$ & $5.27\%$ \\

Online GPR \cite{hu2018three}
& $75.0\%$ & $1.20\times$ & $5.22\%$ \\

FEM-based \cite{saghour2025dual}
& $75.0\%$ & $0.61\times$ & $3.92\%$ \\

\textbf{CADeT (Ours)}
& $\mathbf{98.0\%}$
& $\mathbf{2.2\times}$
& $\mathbf{11.26\%}$ \\
\hline
\end{tabular}
\end{table}

\hl{CADeT also achieved the highest normalized convergence-speed factor
among the evaluated methods.
These comparisons evaluate the complete CADeT framework and do not by
themselves isolate the source of the performance differences.
The contributions of probing actions, probing-action selection, and
the manipulation components are examined separately in}
Sections~\ref{sec:mode_identification},
\ref{sec:probing_policy}, and
\ref{sec:component_ablation}, respectively.

\section{Physical \acs{ramis} Experiments}

\subsection{Experiment Setup}
The experimental validation was conducted on the \ac{dvrk} platform \cite{kazanzides2014open}, using two 7-degree-of-freedom \acp{psm} and one 4-degree-of-freedom \ac{ecm}, as shown in Fig.~\ref{fig:setup}.
The scene was observed through a stereoscopic endoscope providing $1080 \times 720$ RGB video.
To replicate realistic anatomical properties, we used soft-tissue phantoms (HumanX Medical LLC, U.S.) mimicking the viscoelastic behaviour of the liver, colon, and pancreas, alongside ex vivo porcine tissues for biological validation.
\hl{Grasp locations were predefined on accessible proxy-tissue regions to ensure stable grasping and sufficient manipulation range.
For the nominal reachable cases in Figs.}~\ref{fig:pancreas_exp} and \ref{fig:mesentry_exp}, \hl{the desired target configurations were selected within the local deformation range of each grasp.
Additional reachability-limited cases are evaluated separately in Section}~\ref{sec:reachability}.

\hl{At the beginning of each physical trial, the first} $N_0=10$ \hl{probing responses were buffered to initialize} $P_{\mathrm{n}}$ \hl{for the current grasping region.
These responses were generated by the initial probing actions already included in active sensing and required no separate calibration motion.
For the Blocked-mode experiments, initialization was completed before the target reached the external constraint.
The same attenuation factor} $\lambda_{\mathrm{b}}=0.4$ \hl{selected in simulation was used to construct} $P_{\mathrm{stiff}}$ \hl{for all phantom and ex vivo experiments.
Both proxy-response distributions remained fixed within each trial and were reinitialized when the tissue sample or grasping region changed.}

\begin{figure}[!hbt]
    \centering
    \addtolength{\abovecaptionskip}{-15pt}
    \includegraphics[width=1\linewidth]{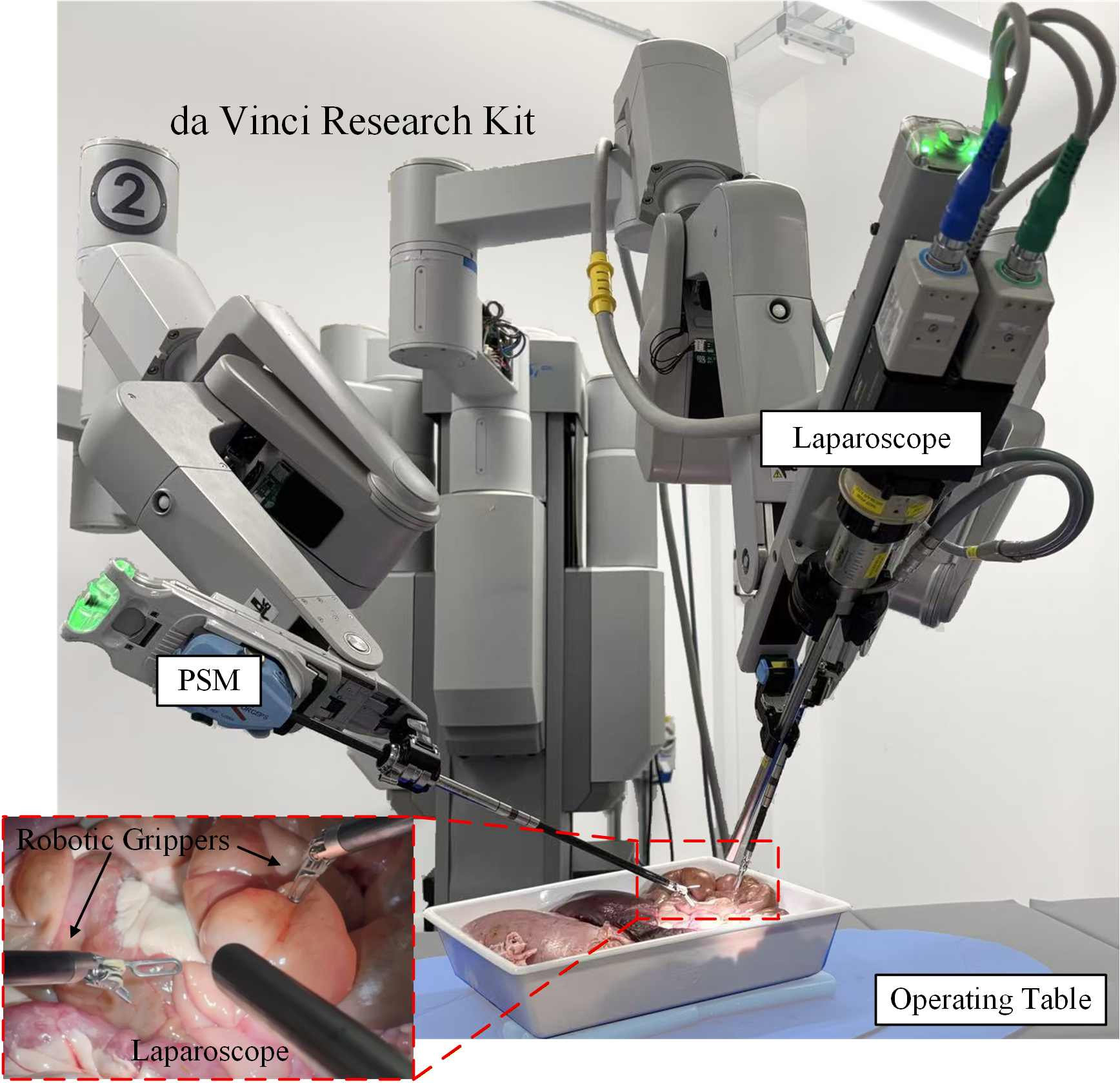}
    \caption{Experimental setup for the \acs{ramis} validation on the \acs{dvrk} platform, detailing the phantom and ex vivo environments.}
    \label{fig:setup}
\end{figure}

\begin{figure}[!hbt]
    \centering
    \addtolength{\abovecaptionskip}{-15pt}
    \input{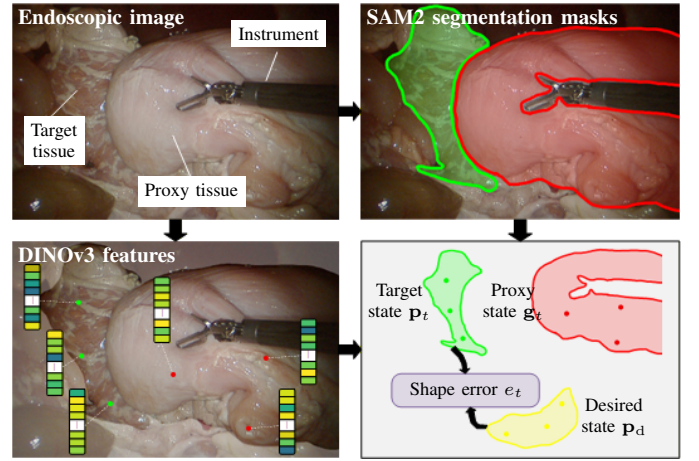}
    \caption{\hl{SAM2--DINOv3 perception pipeline for marker-free tissue segmentation, feature tracking, and shape-state estimation.}}    \label{fig:sam_dino}
\end{figure}

\begin{figure*}[!hbt]
    \centering
    \addtolength{\abovecaptionskip}{-10pt}
    \input{pancreas_exp}
    \caption{The left columns show time-series snapshots of the surgical scene, with the stomach proxy represented by the red boundary and red feature points, the pancreatic target by the green boundary and green feature points, and yellow points denoting the desired target configuration. 
    The right columns show the mode belief $b_t(\tau)$ as a stacked area chart together with the mean target-shape error $e_t$ shown by the red solid line.}
    \label{fig:pancreas_exp}
\end{figure*}

\begin{figure*}[!hbt]
    \centering
    \addtolength{\abovecaptionskip}{-10pt}
    \input{mesentry_exp}
    \caption{The left columns show time-series snapshots of the surgical scene, with the colon proxy represented by the red boundary and red feature points, the mesentery target by the green boundary and green feature points, and yellow points denoting the desired target configuration. 
    The right columns show the mode belief $b_t(\tau)$ as a stacked area chart together with the mean target-shape error $e_t$ shown by the red solid line.}
    \label{fig:mesentry_exp}
\end{figure*}

\hl{To enable marker-free visual state estimation for shape servoing, the perception pipeline combines SAM2} \cite{ravi2024sam} and DINOv3 \cite{simeoni2025dinov3}, as shown in Fig.~\ref{fig:sam_dino}.
\hl{SAM2 is used to segment the proxy and target tissues from each endoscopic frame, after which DINOv3 descriptors are used to track selected tissue feature points across frames.}
\hl{The tracked feature-point locations define the current proxy and target states,} $\mathbf{g}_t$ and $\mathbf{p}_t$, while the desired target state $\mathbf{p}_{\mathrm d}$ \hl{is manually specified before each trial. }
\hl{The corresponding shape error $e_t$ is then provided to the causal-aware deformation transmission controller.}

\subsection{Pancreatic Exposure via Stomach Manipulation}

In pancreatic surgery, adhesions between the stomach and pancreas can make direct pancreatic manipulation undesirable, motivating indirect manipulation through the stomach.
In this experiment, the stomach serves as the proxy, while the pancreas is the \hl{indirectly manipulated target whose visible surface is tracked by the perception system}.

We first validated the framework using phantom stomach and pancreas models coupled by a pressure-sensitive adhesive, followed by ex vivo trials using fresh porcine tissue to evaluate performance under biological heterogeneity.

As shown in Fig.~\ref{fig:pancreas_exp}, \hl{probing reduced the initial mode uncertainty and identified the Coupled Mode.
The deformation transmission controller then used the estimated adhesion Jacobian} $\mathbf{J}_{\mathrm{adh}}$ \hl{to manipulate the stomach and retract the pancreas toward the desired configuration.}

\subsection{Mesentery Manipulation}

The compliance and fragility of the mesentery make direct manipulation challenging, motivating indirect manipulation through the attached intestine.
In this experiment, the colon serves as the proxy, while the mesentery is the indirectly manipulated target.

The shape-control objective was to flatten a folded mesenteric section or retract it to a specified spatial configuration to reveal underlying structures.
In the phantom study, a simplified model was constructed using artificial small and large intestines, with the large intestine serving as the proxy.

As shown in Fig.~\ref{fig:mesentry_exp}, the system successfully completed the indirect manipulation task.
\hl{The mode belief converged to the Coupled Mode, allowing the transmission model to be updated during manipulation.}
The high compliance of the colon resulted in lower \ac{dtr}, requiring larger proxy motions to produce the desired target deformation.

\hl{We further assessed robot-motion settlement.
Over the final} $N_{\rm s}$ \hl{control cycles of each trial that satisfied the convergence criterion, the mean Cartesian displacement command magnitude was}
$0.48 \pm $\SI{0.07}{\milli\meter},
\hl{with residual motion remaining below the predefined settlement threshold} $\epsilon_{\rm u}$.

\subsection{Reachability and Failure Case} \label{sec:reachability}
\hl{To complement the representative trajectories in Figs.}~\ref{fig:pancreas_exp}
\hl{and} \ref{fig:mesentry_exp},
\hl{Table}~\ref{tab:physical_results}
\hl{summarizes the physical experiments in terms of the inferred transmission mode and target-shape error.
The ground-truth mode was determined from the known attachment and blocking conditions.}

\begin{table}[!t]
\centering
\caption{\hl{Quantitative results of the physical RAMIS experiments.}}
\label{tab:physical_results}
\setlength{\tabcolsep}{3.0pt}
\begin{tabular}{lccccc}
\hline
\textbf{Case}
& Tissue
& \makecell{Test\\Condition} 
& \makecell{Mode\\Identified}
& \makecell{Initial\\Error (mm)}
& \makecell{Final\\Error (mm)} \\
\hline

Pancreas (a)
& Ex vivo
& Decoupled
& $\checkmark$
& \num{39.8}
& \num{36.6} \\

Pancreas (b)
& Ex vivo
& Blocked
& $\checkmark$
& \num{57.6}
& \num{2.8} \\

Pancreas (c)
& Phantom
& Coupled
& $\checkmark$
& \num{45.3}
& \num{2.3} \\

Pancreas (d)
& Phantom
& Decoupled
& $\checkmark$
& \num{52.1}
& \num{53.5} \\

Mesentery (a)
& Ex vivo
& Coupled
& $\checkmark$
& \num{36.8}
& \num{2.7} \\

Mesentery (b)
& Ex vivo
& Decoupled
& $\checkmark$
& \num{43.8}
& \num{43.8} \\

Mesentery (c)
& Phantom
& Coupled
& $\checkmark$
& \num{38.0}
& \num{2.1} \\

Mesentery (d)
& Phantom
& Blocked
& $\checkmark$
& \num{42.5}
& \num{1.9} \\

\hline
\multicolumn{6}{l}{\hl{\textit{Additional constrained/reachability cases}}} \\

Pancreas (e1)
& Ex vivo
& Coupled
& $\checkmark$
& \num{45.8}
& \num{16.5} \\

Pancreas (e2)
& Phantom
& Coupled
& $\checkmark$
& \num{37.8}
& \num{21.5} \\

Mesentery (e1)
& Ex vivo
& Coupled
& $\checkmark$
& \num{57.1}
& \num{32.5} \\

Mesentery (e2)
& Phantom
& Blocked
& $\checkmark$
& \num{50.2}
& \num{34.1} \\

\hline
\end{tabular}
\end{table}

\hl{The expected transmission condition was identified in all} \num{12} \hl{cases, with Blocked cases counted as correct when evidence for the Blocked Mode was observed during local constraint engagement rather than throughout the entire trial.}
\hl{For the reachable Coupled cases in Figs.}~\ref{fig:pancreas_exp} \hl{and} \ref{fig:mesentry_exp},
\hl{the mean shape error decreased from} $36.8$--\SI{45.3}{\milli\meter}
\hl{to} $2.1$--\SI{2.7}{\milli\meter}, \hl{whereas the Decoupled cases showed little or no error reduction because deformation was not transmitted to the target.}

\hl{The additional cases were intentionally selected to examine limitations caused by local reachability and external constraints.
In the three Coupled cases, the transmission mode was correctly identified and the shape error decreased, but residual errors of} \num{16.5}--\SI{32.5}{\milli\meter} \hl{remained because the requested target configurations were outside the locally achievable deformation range under the current grasp.
The Blocked case similarly retained a residual error of} \SI{34.1}{\milli\meter} under the external constraint. These results distinguish limited target reachability from failure of mode inference.

\section{Discussion}
\hl{Evaluations in simulation and on the dVRK support the feasibility of CADeT for indirect soft-tissue manipulation under the tested conditions.
The simulation ablations show that additional probing helps resolve the observational ambiguity between the Decoupled and Blocked modes, which may produce similar target-motion responses during task-directed control.
Random probing could also provide useful mode information but showed greater direction-dependent variability, whereas information-guided probing improved the consistency and efficiency of mode inference by favouring actions with more distinguishable predicted responses.}

\hl{The shape-control results further suggest that CADeT benefits from conditioning target prediction on the inferred transmission mode.
When the belief concentrates on the Coupled Mode, the controller uses the learned adhesion Jacobian to predict target deformation; as the Coupled-Mode belief decreases, the belief-weighted transmission gain is reduced, discouraging corrective actions with limited predicted effect on the target.
This behaviour is consistent with the higher task-success rates and faster convergence observed in the component ablations and baseline comparisons.}

\hl{Target-shape convergence is distinguished from robot-motion settlement.
Because perception noise, online model updates, and viscoelastic relaxation can induce residual corrections after the target error enters the desired range, task completion additionally requires the commanded end-effector increment to satisfy the settlement criterion in} \eqref{eq:settlement_condition} for $N_{\rm s}$ \hl{consecutive control cycles, after which the robot holds its current pose.
This criterion demonstrates closed-loop settlement under the tested conditions rather than a general guarantee of asymptotic stability.}

Despite these advantages, the framework has several limitations.
\hl{Mode-inference accuracy depends on the strength of the physical coupling between the proxy and target; when adhesion is extremely weak, transmitted deformation may become indistinguishable from measurement noise, resulting in slower convergence or misidentification.
The current framework also assumes a predefined grasp and does not address grasp planning or arbitrary target reachability.
Even in the Coupled Mode, achievable target deformation is limited by the current grasp and local transmission characteristics, potentially leaving residual error for unreachable targets.
The local transmission model relies on a first-order approximation and may become inaccurate under large deformations or non-reversible tissue changes.
The framework also provides no formal tissue-safety guarantee because contact force and tissue strain are neither directly measured nor explicitly constrained; such biomechanical constraints would be required for safety-critical clinical deployment.
Finally, performance under dynamic environmental constraints and unpredictable external disturbances remains to be investigated.}

\section{Conclusion}
This paper introduced a causal-aware framework for the indirect manipulation of deformable anatomy in \ac{ramis}.
By formulating the interaction as a \ac{scm} and using active sensing to resolve ambiguity in transmission-mode inference, we demonstrated that \hl{a robot can indirectly control a partially visible target that is inaccessible to direct manipulation by actuating an accessible proxy tissue.}
The proposed deformation transmission control strategy, CADeT, integrates mode inference with belief-aware model predictive control.
Validation in simulation and on the dVRK using phantom and ex vivo porcine tissues showed that the framework can distinguish competing transmission modes and achieve indirect shape control for locally reachable target configurations under the evaluated conditions.
Future research will focus on extending the framework to more complex multi-organ transmission chains and incorporating real-time force feedback to complement the visual perception loop.

% \section*{Acknowledgment}
% All the experiments involving human cadaveric tissues were performed under ethical approval from the University of Leeds. 
% The authors would like to thank Intuitive Surgical, Inc. for the donation of the da Vinci system, the STORM Lab technician, Samwise Wilson, for hardware support, and the anatomy facilities technicians of the School of Medicine, Sarah Wilson and Charlotte Coleman, for their support in the cadaveric experiments.


\ifCLASSOPTIONcaptionsoff
  \newpage
\fi



\bibliographystyle{IEEEtran}
% \bibliography{IEEEabrv,ref}
\bibliography{ref}


% \end{thebibliography}

\appendix[Proof of Proposition \ref{prop:indistinguishability}]
\begin{proof}\label{app:proof_prop1}
In the stationary regime $\mathcal{S}_0$, the target response provides no discriminative information because
$\Delta\mathbf{p}_t \approx \mathbf{0}$
under both the Decoupled and Blocked Modes.
According to \eqref{eq:L=0} and \eqref{eq:L=2}, the two modes therefore differ only in their proxy-response distributions,
\[
P_{\mathrm{n}}(\Delta\mathbf{g}_t\mid\mathbf{u}_t)
=
\mathcal{N}\!\left(
\Delta\mathbf{g}_t;
\boldsymbol{\mu}_{\mathrm{n}}(\mathbf{u}_t),
\boldsymbol{\Sigma}_{\mathrm{n}}
\right)
\]
and
\[
P_{\mathrm{stiff}}(\Delta\mathbf{g}_t\mid\mathbf{u}_t)
=
\mathcal{N}\!\left(
\Delta\mathbf{g}_t;
\lambda_{\mathrm{b}}
\boldsymbol{\mu}_{\mathrm{n}}(\mathbf{u}_t),
\boldsymbol{\Sigma}_{\mathrm{n}}
\right).
\]

Since the two Gaussian distributions share the same covariance,
their KL divergence is
\[
D_{\mathrm{KL}}
=
\frac{1}{2}
(1-\lambda_{\mathrm{b}})^2
\left\|
\boldsymbol{\mu}_{\mathrm{n}}(\mathbf{u}_t)
\right\|_{\boldsymbol{\Sigma}_{\mathrm{n}}^{-1}}^2 .
\]

The passive policy $\pi_{\mathrm{pass}}$ minimizes the tracking error together with the control effort.
In $\mathcal{S}_0$, the predicted target response is negligible under both competing modes, so the control-effort regularization drives
$\mathbf{u}_t\rightarrow\mathbf{0}$.
Assuming the nominal proxy-response model is continuous and satisfies
$\boldsymbol{\mu}_{\mathrm{n}}(\mathbf{0})=\mathbf{0}$,
it follows that
\(
D_{\mathrm{KL}}
\rightarrow 0\).

Therefore, passive observations become asymptotically non-discriminative between the Decoupled and Blocked Modes, which establishes the observational indistinguishability stated in Proposition~\ref{prop:indistinguishability}.
\end{proof}

\end{document}